\documentclass[notitlepage]{article}
\usepackage[round,comma]{natbib}
\usepackage{amsthm}
\usepackage{sectsty}
\newtheorem{setting}{Setting}

\usepackage{bold-extra}
\usepackage{amssymb}
\usepackage{amsmath}
\usepackage{xspace}
\usepackage{url}
\usepackage{color}
\usepackage{enumerate}
\usepackage{comment}
\usepackage{algorithm}
\usepackage{algorithmic}
\usepackage{mdframed}
\usepackage{anysize}
\usepackage{mathtools}
\usepackage{float}
\marginsize{3cm}{4cm}{3cm}{4cm}

\DeclareMathOperator*{\argmin}{arg\,min}

\newcommand{\eqr}[1]{Eq.~\eqref{eq:#1}}
\newcommand {\norm}[1]{\ensuremath{\| #1 \|}}
\newcommand {\dnorm}[1]{\ensuremath{\| #1 \|_\star}}
\newcommand {\set}[1]{\ensuremath{\left\{#1\right\}}}
\newcommand{\BO}{\mathcal{O}}
\newcommand{\R}{\ensuremath{\mathbb{R}}}
\newcommand{\grad}{\triangledown}
\newcommand{\abs}[1]{|#1|}
\DeclareMathOperator{\dom}{dom}
\DeclareMathOperator{\diag}{diag}
\newcommand{\qq}[1]{\qquad \text{#1} \qquad}

\newcommand{\qqand}{\qquad \text{and} \qquad}
\newcommand{\qqwhere}{\qquad \text{where} \qquad}

\newcommand{\Regret}{\operatorname{Regret}}

\newcommand{\inv}{^{-1}}

\newcommand{\h}{\frac{1}{2}}

\makeatletter
\@ifundefined{theorem}{%
\newtheorem {theorem}{Theorem}}{}
\@ifundefined{lemma}{%
\newtheorem {lemma}[theorem]{Lemma}}{}
\@ifundefined{corollary}{%
\newtheorem {corollary}[theorem]{Corollary}}{}

\newcommand{\defn}[1]{\emph{#1}}

\newcommand{\loss}{\ell}

\definecolor{darkgreen}{rgb}{0,0.4,0.0}
\definecolor{darkblue}{rgb}{0,0.1,0.3}
\definecolor{darkred}{rgb}{0.7,0.0,0.0}

\newcommand{\xs}{x^*}
\newcommand{\rc}{r^\star}

\newcommand{\hf}{\bar{f}}  %

\newcommand{\dpsi}{\psi^\star}
\newcommand{\bPsi}{\bar{\Psi}}

\newcommand{\normt}[1]{\norm{#1}_{(t)}}
\newcommand{\normtm}[1]{\norm{#1}_{(t-1)}}

\newcommand{\dnormt}[1]{\norm{#1}_{(t),\star}}
\newcommand{\dnormtm}[1]{\norm{#1}_{(t-1),\star}}

\newcommand{\ti}{_{t+1}}
\renewcommand{\tt}{_{1:t}}
\newcommand{\ttm}{_{1:t-1}}
\newcommand{\tT}{_{1:T}}

\newcommand{\ztt}{_{0:t}}
\newcommand{\zttm}{_{0:t-1}}
\newcommand{\ztT}{_{0:T}}
\newcommand{\ztTm}{_{0:T-1}}

\newcommand{\RS}{\R \cup \set{\infty}}
\newcommand{\X}{\mathcal{X}}
\newcommand{\ind}{I}
\newcommand{\indX}{\ind_{\X}}
\newcommand{\proj}{\Pi}
\newcommand{\projX}{\proj_{\X}}

\newcommand{\ha}{\phi_1}
\newcommand{\hb}{\phi_2}
\newcommand{\phd}{\phi_1^\star}

\newcommand{\hx}{\hat{x}}
\newcommand{\B}{\mathcal{B}}
\newcommand{\rp}{r^\B}

\newcommand{\gp}{g^{(\Psi)}}
\newcommand{\gr}{g^{(r)}}

\newcommand{\titleA}{A Survey of Algorithms and Analysis}
\newcommand{\titleB}{for Adaptive Online Learning}

\newcommand{\DA}{Dual Averaging\xspace}
\newcommand{\RDA}{Regularized Dual Averaging\xspace}
\newcommand{\MD}{Mirror Descent\xspace}
\newcommand{\OGD}{Online Gradient Descent\xspace}
\newcommand{\NaturalProx}{Native FTRL-Proximal\xspace}
\newcommand{\Natural}{Native FTRL\xspace}
\newcommand{\WLOG}{w.l.o.g.\@\xspace}
\newcommand{\wrt}{w.r.t.\@\xspace}

\renewcommand{\ng}{\theta} %
\newcommand{\RR}{R}
\newcommand{\RD}{R^\star}
\newcommand{\Equiv}{$\Leftrightarrow$\xspace}

\newcommand{\tabeq}[1]{{$\begin{aligned} #1 \end{aligned}$}}

\newcommand{\comm}[1]{$\quad /\!/$\emph{#1}}

\newcommand{\jmlronly}[1]{}

\sectionfont{\fontsize{12}{15}\selectfont}  %
\subsectionfont{\fontsize{11}{15}\selectfont}  %

\begin{document}

\title{\titleA\\ \titleB}
\author{H. Brendan McMahan\\
Google, Inc. \\
\texttt{mcmahan@google.com}}
\date{}
\maketitle

\begin{abstract}
  We present tools for the analysis of Follow-The-Regularized-Leader
  (FTRL), \DA, and \MD algorithms when the regularizer (equivalently,
  prox-function or learning rate schedule) is chosen adaptively based
  on the data.  Adaptivity can be used to prove regret bounds that
  hold on every round,
  and also allows for data-dependent regret bounds as in AdaGrad-style
  algorithms (e.g., \OGD with adaptive
  per-coordinate learning rates).  We present results from a large
  number of prior works in a unified manner, using a modular and tight
  analysis that isolates the key arguments in easily re-usable lemmas.
  This approach strengthens previously known FTRL analysis techniques
  to produce bounds as tight as those achieved by potential functions
  or primal-dual analysis.  Further, we prove a general and exact
  equivalence between an arbitrary adaptive \MD algorithm and a
  corresponding FTRL update, which allows us to analyze any \MD
  algorithm in the same framework.  The key to bridging the gap
  between \DA and \MD algorithms lies in an analysis of
  the FTRL-Proximal algorithm family.
  Our regret bounds are proved in the most general form, holding for
  arbitrary norms and non-smooth regularizers with time-varying
  weight.
\end{abstract}

\jmlronly{
}

\section{Introduction}
We consider the problem of online convex optimization over a series of
rounds $t \in \{1, 2, \dots\}$.  On each round the algorithm selects a
point (e.g., a predictor or an action) $x_t \in \R^n$, and then an
adversary selects a convex loss function $f_t$, and the algorithm
suffers loss $f_t(x_t)$.  The goal is to minimize
\begin{equation}\label{eq:regret}
  \Regret_T(\xs, f_t) \equiv
    \sum_{t=1}^T f_t(x_t) - \sum_{t=1}^T f_t(\xs),
\end{equation}
the difference between the algorithm's loss and the loss of a fixed
point $\xs$, potentially chosen with full knowledge of the sequence of
$f_t$ up through round $T$. When the functions $f_t$ and round $T$
are clear from context we write $\Regret(\xs)$. The ``adversary''
choosing the $f_t$ need not be malicious, for example the $f_t$ might
be drawn from a distribution.
The name ``online convex optimization'' was introduced by \citet{zinkevich03giga}, though the setting was introduced earlier by \citet{gordon99regret}.
When a particular set of comparators $\X$ is fixed in advance, one is
often interested in $\Regret(\X) \equiv \sup_{\xs \in \X}
\Regret(\xs)$; since $\X$ is often a norm ball, frequently we bound
$\Regret(\xs)$ by a function of $\norm{\xs}$.  

Online algorithms with good regret bounds (that is, bounds that are
sublinear in $T$) can be used for a wide variety of prediction and
learning tasks~\citep{cesabianchi06plg,shwartz12online}.  The case of
online logistic regression, where one predicts the probability of a
binary outcome, is typical.  Here, on each round a feature vector $a_t
\in \R^n$ arrives, and we make a prediction $p_t = \sigma(a_t \cdot
x_t) \in (0, 1)$ using the current model coefficients $x_t \in \R^n$,
where $\sigma(z) = 1/ (1 + e^{-z})$.  The adversary then reveals the
true outcome $y_t \in \set{0, 1}$, and we measure loss with the
negative log-likelihood, $\loss(p_t, y_t) = -y_t \log p_t - (1-y_t)
\log (1 - p_t)$.  We encode this problem as online convex optimization
by taking $f_t(x) = \loss(\sigma(a_t \cdot x), y_t)$; these $f_t$ are
in fact convex.  Linear Support Vector Machines (SVMs), linear
regression, and many other learning problems can be encoded in a
similar manner; \citet{shwartz12online} and many of the other works
cited here contain more details and examples.

\begin{algorithm}[t]
\caption{General Template for Adaptive FTRL}\label{alg:adaptive}
\begin{algorithmic}
\STATE \textbf{Parameters:} Scheme for selecting convex $r_t$ s.t. $\forall x,\ r_t(x) \ge 0$ for $t=0,1,2, \dots$
\STATE $x_1 \leftarrow \argmin_{x \in \R^n}\ r_0(x)$
\FOR{ $t = 1, 2, \ldots$}
\STATE Observe convex loss function $f_t: \R^n \rightarrow \R \cup \{\infty\}$
\STATE Incur loss $f_t(x_t)$
\STATE Choose incremental convex regularizer $r_t$, possibly based on $f_1, \dots f_t$
\STATE Update \[x\ti \leftarrow \argmin_{x \in \R^n}\ \sum_{s=1}^t f_s(x) + \sum_{s=0}^t  r_s(x)\]
\ENDFOR
\end{algorithmic}
\end{algorithm}

We consider the family of Follow-The-Regularized-Leader (FTRL, or
FoReL) algorithms as shown in Algorithm~\ref{alg:adaptive}
\citep{shwartz07thesis,shalev07primaldual,rakhlin08notes,mcmahan10boundopt,mcmahan10equiv}. \citet{shwartz12online} and \citet{hazan15survey} provide a
comprehensive survey of analysis techniques for non-adaptive members
of this algorithm family, where the regularizer is fixed for all
rounds and chosen with knowledge of $T$.  In this survey, we allow the
regularizer to change adaptively over the course of an unknown-horizon
game.  Given a sequence of incremental regularization functions $r_0, r_1, r_2,
\dots$, we consider the algorithm that selects
\begin{align}\label{eq:update}
x_1 &\in \argmin_{x \in \R^n} r_0(x) \notag\\
x\ti & = \argmin_{x \in \R^n} f\tt(x) + r\ztt(x) \qquad \text{for $t=1, 2, \dots$},
\end{align}
where we use the compressed summation notation $f\tt(x) = \sum_{s=1}^t
f_s(x)$ (we also use this notation for sums of scalars or vectors).
The argmin in \eqr{update} is over all $\R^n$, but it is often
necessary to constrain the selected points $x_t$ to a convex feasible
set $\X$. This can be accomplished in our framework by including the
indicator function $\indX$ as a term in $r_0$ ($\indX$ is a convex
function defined by $\indX(x) = 0$ for $x \in \X$ and $\infty$
otherwise); details are given in Section~\ref{sec:feasible}.  The
algorithms we consider are adaptive in that each $r_t$ can be chosen
based on $f_1, f_2, \dots, f_t$.  For convenience, we define functions
$h_t$ by
\begin{align*}
h_0(x) &= r_0(x) \\
h_t(x) &= f_t(x) + r_t(x) \qquad \text{for $t=1, 2, \dots$}
\end{align*}
so $x\ti = \argmin_x h\ztt(x)$.  Generally we will assume the $f_t$
are convex, and the $r_t$ are chosen so that $r\ztt$ (or $h\ztt$) is
strongly convex for all $t$, e.g., $r\ztt(x) = \frac{1}{2
  \eta_t}\norm{x}_2^2$ (see Sections \ref{sec:genregret} and
\ref{sec:convex} for a review of important definitions and results
from convex analysis).

FTRL algorithms generalize the Follow-The-Leader (FTL) approach
\citep{hannan57ftl,kalai03ftpl}, which selects $x\ti = \argmin_x
f\tt(x)$. FTL can provide sublinear regret in the case of strongly convex
functions (as we will show), but for general convex functions
additional regularization is needed.

Adaptive regularization can be used to construct practical algorithms
that provide regret bounds that hold on all rounds $T$, rather than
only on a single round $T$ which is chosen in advance. The
framework is also particularly suitable for analyzing AdaGrad-style
algorithms that adapt their regularization or norms based on the
observed data, for example those of \citet{mcmahan10boundopt} and
\cite{duchi10adaptive,duchi11adaptivejournal}.
This approach leads to regret bounds that
depend on the actual observed sequence of gradients $g_t$, rather than
bounds in terms of the number of rounds $T$ and the worst-case
magnitude of the gradients $G$, e.g., terms like $\sqrt{\sum_{t=1}^T
  g_t^2}$ rather than $G\sqrt{T}$. These tighter bounds translate to
much better performance in practice, especially for high-dimensional
but sparse problem (e.g., bag-of-words feature vectors). Examples of
such algorithms are analyzed in Sections~\ref{sec:adagradprox}
and~\ref{sec:adagraddual}.

We also study \MD algorithms, for example updates like
\[
x\ti = \argmin_{x \in \X} g_t \cdot x + \lambda \norm{x}_1 + \frac{1}{2\eta_t}\norm{x - x_t}_2^2
\]
for functions $f_t(x) = g_t \cdot x + \lambda \norm{x}_1$, where
$\eta_t$ is an adaptive learning rate. This update generalizes \OGD
with a non-smooth regularization term; \MD also encompasses the use of
an arbitrary Bregman divergence in place of the $\norm{\cdot}_2^2$
penalty above. We will discuss this family of algorithms at length in
Section~\ref{sec:md}. In fact, \MD algorithms can all be expressed as
particular members of the FTRL family, though generally not
the most natural ones. In particular, since the state maintained by
\MD is essentially only the current feasible point $x_t$, we will see
that \MD algorithms are forced to linearize penalties like $\lambda
\norm{x}_1$ from previous rounds, while the more natural FTRL
algorithms can keep these terms in closed form, leading to practical
advantages such as producing sparser models when $L_1$ regularization
is used.

While we focus on online algorithms and regret bounds, the
development of many of the algorithms considered rests heavily on work
in general convex optimization and stochastic optimization.  As a few
starting points, we refer the reader to \citet{nemirovski83} and
\citet{nesterov04book,nesterov07composite}.  
Going the other way, the algorithms presented here can be applied
to batch optimization problems of the form
\begin{equation}\label{eq:batch}
\argmin_{x \in \R^n}\ F(x) \qqwhere F(x) \equiv \sum_{t=1}^T f_t(x)
\end{equation}
by running the online algorithm for one or more passes over the set of
$f_t$ and returning a suitable point (usually the last $x_t$ or an
average of past $x_t$).  Using online-to-batch conversion
techniques (e.g., \citet{cesabianchi04generalization}, \citet[Chapter
5]{shwartz12online}), one can convert the regret bounds given here to
convergence bounds for the batch problem.  Many state-of-the-art
algorithms for batch optimization over very large datasets can be
analyzed in this fashion.

\paragraph{Outline} In Section~\ref{sec:ftrl}, we elaborate on the
family of algorithms encompassed by the update of \eqr{update}.  We then
state two very general regret bounds, Theorems~\ref{thm:centered} and
\ref{thm:proximal}.  While these results are not completely new, they are stated in
enough generality to cover many known results for general and strongly
convex functions; in Section~\ref{sec:applications} we use them to
derive concrete bounds for many standard online algorithms.

In Section~\ref{sec:genproof} we break the analysis of adaptive FTRL
algorithms into three main components, which helps to modularize the
arguments.  In Section~\ref{sec:inductive} we prove the \emph{Strong
  FTRL Lemma} which lets us express the regret through round $T$ as a
regularization term on the comparator $\xs$, namely $r\ztT(\xs)$, plus
a sum of per-round stability terms.  This reduces the problem of
bounding regret to that of bounding these per-round terms.  In
Section~\ref{sec:convex} we review some standard results from convex
analysis, and prove lemmas that make bounding the per-round terms
relatively straightforward.  The general regret bounds are then proved
in Section~\ref{sec:proofs} as corollaries of these results.

Section~\ref{sec:composite} considers the special case of a composite
objective, where for example $f_t(x) = \ell_t(x) + \Psi(x)$ with
$\ell_t$ is a smooth loss on the $t$'th training example and $\Psi$ is
a possibly non-smooth regularizer (e.g., $\Psi(x) =
\norm{x}_1$). Finally, Section~\ref{sec:md} proves the equivalence of
an arbitrary adaptive \MD algorithm and a certain FTRL algorithm, and
uses this to prove regret bounds for \MD.

\paragraph{Summary of Contributions}
A principal goal of this work is to provide a useful summary of
central results in the analysis of adaptive algorithms for online
convex optimization; whenever possible we provide precise references
to earlier results that we re-prove or strengthen.
Achieving this goal in a concise fashion requires some new results,
which we summarize here.

The FTRL style of analysis is both modular and intuitive, but in
previous work resulted in regret bounds that are not the tightest
possible; we remedy this by introducing the Strong FTRL Lemma in
Section~\ref{sec:inductive}.  This also relates the FTRL analysis
technique to the primal-dual style of analysis.

By analyzing both FTRL-Proximal algorithms (introduced in the next
section) and \DA algorithms in a unified manner, it is much easier to contrast
the strengths and weaknesses of each approach.  This highlights a
technical but important ``off-by-one'' difference between the two
families in the adaptive setting, as well as an important difference
when the algorithm is unconstrained (any $x_t \in \R^n$ is feasible).

Perhaps the most significant new contribution is given in
Section~\ref{sec:md}, where we show that \emph{all} \MD algorithms
(including adaptive algorithms for composite objectives) are in fact
particular instances of the FTRL-Proximal algorithm schema, and can be
analyzed using the general tools developed for the analysis of FTRL.

\begin{algorithm}[t]
\caption{General Template for Adaptive Linearized FTRL}\label{alg:linearized}
\begin{algorithmic}
\STATE \textbf{Parameters:} Scheme for selecting convex $r_t$ s.t. $\forall x,\ r_t(x) \ge 0$ for $t=0,1,2, \dots$
\STATE $z \leftarrow \mathbf{0} \in \R^n$  \comm{Maintains $g_{1:t}$}
\STATE $x_1 \leftarrow \argmin_{x \in \R^n}\ z \cdot x + r_0(x)$
\FOR{ $t = 1, 2, \ldots$}
\STATE Select $x_t$, observe loss function $f_t$, incur loss $f_t(x_t)$
\STATE Compute a subgradient $g_t \in \partial f_t(x_t)$
\STATE Choose incremental convex regularizer $r_t$, possibly based on $g_1, \dots, g_t$
\STATE $z \leftarrow z + g_t$
\STATE $x\ti \leftarrow \argmin_{x \in \R^n}\ z \cdot x + r\ztt(x)$ %
\comm{Often solved in closed form}
\ENDFOR
\end{algorithmic}
\end{algorithm}

\section{The FTRL Algorithm Family and General Regret Bounds}
\label{sec:ftrl}
We begin by considering two important dimensions in the space of FTRL
algorithms.  First, the algorithm designer has significant flexibility
in deciding whether the sum of previous loss functions is optimized
exactly as $f\tt(x)$ in \eqr{update}, or if the true losses should be
replaced by appropriate lower bounds, $\hf\tt(x)$, for computational
efficiency.  Second, we consider whether the incremental regularizers
$r_t$ are all minimized at a fixed stationary point $x_1$, or
are chosen so they are minimized at the current $x_t$.  After
discussing these options, we state general regret bounds.

\subsection{Linearization and the Optimization of Lower Bounds}%
In practice, it may be infeasible to solve the optimization problem of
~\eqr{update}, or even represent it as $t$ becomes sufficiently large.
A key point is that we can derive a wide variety of first-order
algorithms by linearizing the $f_t$, and running the algorithm on
these linear functions.  Algorithm~\ref{alg:linearized} gives the
general scheme.
For convex $f_t$, let $x_t$ be defined as above, and let $g_t
\in \partial f_t(x_t)$ be a subgradient (e.g., $g_t = \grad f_t(x_t)$
for differentiable $f_t$).  Then, a key observation of
\citet{zinkevich03giga} is that convexity implies for any comparator
$\xs$, $f_t(x_t) - f_t(\xs) \leq g_t \cdot (x_t - \xs)$.  Thus, if we let
$\hf_t(x) = g_t \cdot x$, then for any algorithm the regret against
the functions $\hf_t$ upper bounds the regret against the original
$f_t$:
\[
\Regret(\xs, f_t) \le \Regret(\xs, \hf_t).
\]
  Note we can construct the functions $\hf_t$ on the fly (after
observing $x_t$ and $f_t$) and then present them to the algorithm.  

Thus, rather than solving $x\ti = \argmin_x f\tt(x) + r\ztt(x)$ on
each round $t$, we now solve $x\ti = \argmin_x g\tt\cdot x +
r\ztt(x)$. Note that $g\tt \in \R^n$, and we will generally choose the
$r_t$ so that $r\ztt(x)$ can also be represented in constant
space. Thus, we have at least ensured our storage requirements stay
constant even as $t \rightarrow \infty$. Further, we will usually be
able to choose $r_t$ so the optimization with $g\tt$ can be solved in
closed form.  For example, if we take $r\ztt(x) = \frac{1}{2
  \eta}\norm{x}_2^2$ then we can solve $x\ti = \argmin_x g\tt\cdot x +
r\ztt(x)$ in closed form, yielding $x\ti = - \eta g\tt$ (that is, this
FTRL algorithm is exactly constant learning rate \OGD).

However, we will usually state our results in terms of general $f_t$,
since one can always simply take $f_t = \hf_t$ when appropriate.  In
fact, an important aspect of our analysis is that it does not depend
on linearization; our regret bounds hold for the the general update of
\eqr{update} as well as applying to linearized variants.

More generally, we can run the algorithm on any $\hf_t$
that satisfy $\hf_t(x_t) - \hf_t(\xs) \ge f_t(x_t) - f_t(\xs)$ for all
$\xs$ and have the regret bound achieved for the $\hf$ also apply to
the original $f$.  This is generally accomplished by constructing a
lower bound $\hf_t$ that is tight at $x_t$, that is $\hf_t(x) \le f_t(x)$ for
all $x$ and further $\hf_t(x_t) = f_t(x_t)$.  A tight linear lower bound is always
possible for convex functions, but for example if the $f_t$ are all
strongly convex, better algorithms are possible by taking $\hf_t$ to be an
appropriate quadratic lower bound.

A more in-depth introduction to the linearization of convex function
can be found in \citet[Sec 2.4]{shwartz12online}. We also note that
the idea of replacing the loss function on each round with an appropriate lower bound (``linearization of convex functions'') is distinct from the modeling decision to replace a non-convex loss function (e.g., the zero-one loss for classification) with a convex upper bound (e.g., the hinge loss). This ``convexification by surrogate loss'' approach is described in detail by \cite[Sec 2.1]{shwartz12online}.

\subsection{Regularization in FTRL Algorithms}\label{sec:reg}

The term ``regularization'' can have multiple meanings, and so in this
section we clarify on the different roles regularization plays in the
present work.

We refer to the functions $r\ztt$ as regularization functions, with
$r_t$ the incremental increase in regularization on round $t$ (we
assume $r_t(x) \ge 0$).  This is the regularization in the name
Follow-The-Regularized-Leader, and these $r_t$ terms should be viewed
as part of the algorithm itself --- analogous (and in some cases
exactly equivalent) to the learning rate schedule in an \OGD algorithm,
for example.  The adaptive choice of these regularizers is the
principle topic of the current work. We study two main classes of
regularizers: \vspace{-0.1in}
\begin{itemize}
\item In \emph{FTRL-Centered} algorithms, each $r_t$ (and
  hence $r\ztt$) is minimized at a fixed point, $x_1 = \argmin_x
  r_0(x)$. An example is \DA (which also linearizes the losses), where
  $r\ztt$ is called the \emph{prox-function}
  \citep{nesterov09dualaveraging}.

\item In \emph{FTRL-Proximal} algorithms, each incremental
  regularization function $r_t$ is minimized by $x_t$, and we
  call such $r_t$ incremental proximal regularizers.  
\end{itemize}
When we make neither a proximal nor centered assumption on the $r_t$,
we refer to general FTRL algorithms. Theorem~\ref{thm:centered}
(below) allows us to analyze regularization choices that do not fall
into either of these two categories, but the Centered and Proximal
cases cover the algorithms of practical interest.

There are a number of reasons we might wish to add additional
regularization terms to the objective function in the FTRL update. In
many cases this is handled immediately by our general theory by
grouping the additional regularization terms with either the $f_t$ or the
$r_t$. However, in some cases it will be advantageous to handle this
additional regularization more explicitly. We study this situation in
detail in Section~\ref{sec:composite}.

\subsection{General Regret Bounds}\label{sec:genregret}
In this section we introduce two general regret bounds that can be
used to analyze many different adaptive online algorithms. First, we
introduce some additional notation and definitions.  

\paragraph{Notation and Definitions}
An extended-value convex function $\psi : \R^n \rightarrow
\RS$ satisfies
\[
\psi(\theta x + ( 1- \theta) y) \le \theta \psi(x) + (1-\theta)\psi(y),
\]
for $\theta \in (0, 1)$, and the domain of $\psi$ is the convex set $\dom
\psi \equiv \{x : \psi(x) < \infty\}$ (e.g., \citet[Sec. 3.1.2]{boyd}); $\psi$
is proper if $\exists x \in \R^n\ \text{s.t.}\ \psi(x) < +\infty$ and
$\forall x \in \R^n, \psi(x) > -\infty$. We refer to extended-value
proper convex functions as simply ``convex functions.''

We write $\partial \psi(x)$ for the subdifferential of $\psi$ at $x$;
a subgradient $g \in \partial \psi(x)$ satisfies 
\[\forall y \in \R^n,\ \psi(y) \ge \psi(x) + g\cdot (y-x).
\] 
The subdifferential $\partial \psi(x)$ for a convex $\psi$ is
always non-empty for $x \in \text{int}\,(\dom \psi)$, and typically non-empty for any $x \in \dom \psi$ for the functions $\psi$ considered in this work; $\partial \psi(x)$ is empty for $x \not \in \dom \psi$ \citep[Thm. 23.2]{rockafellar}.

Working with extended convex functions lets us encode constraints seamlessly by using $\indX$, the indicator function on a convex set $\X \subseteq \R^n$ given by
\begin{equation}\label{eq:indX}
\indX(x) = \begin{cases}
  0 & x \in \X \\
 \infty & \text{otherwise}~,
\end{cases}
\end{equation}
since $\indX$ is itself an extended convex function. Generally we assume $\X$ is a closed convex set. This approach makes it convenient to write $\argmin_x$ as shorthand for $\argmin_{x \in \R^n}$.

A function $\psi : \R^n \rightarrow \RS$ is
$\sigma$\defn{-strongly convex} \wrt a norm $\norm{\cdot}$ if
for all $x, y \in \R^n,$
\begin{equation}\label{eq:sc}
\forall g \in \partial \psi(x),\ \ \psi(y) \geq \psi(x) + g \cdot (y - x) + \tfrac{\sigma}{2} \norm{y -x}^2.
\end{equation}
If some $\psi$ only satisfies \eqr{sc} for $x,y \in \X$ for
a convex set $\X$, then the function $\psi' = \psi + \indX$
satisfies \eqr{sc} for all $x,y \in \R^n$, and so is strongly convex
by our definition. Thus, we can work with $\psi'$ without any need to
explicitly refer to $\X$.

The \defn{convex conjugate} (or Fenchel conjugate) of an arbitrary
function $\psi : \R^n \rightarrow \RS$ is
\begin{equation}\label{eq:convexconj}
\dpsi(g) \equiv \sup_x g\cdot x - \psi(x).
\end{equation}
For a norm $\norm{\cdot}$, the dual norm is given by
\[
\dnorm{x} \equiv \sup_{y : \norm{y} \leq 1} x \cdot y.
\]
It follows from this definition that for any $x,y \in \R^n$, $x \cdot
y \le \norm{x} \dnorm{y}$, a generalization of H{\"o}lder's
inequality.  We make heavy use of norms $\normt{\cdot}$ that change as
a function of the round $t$; the dual norm of $\normt{\cdot}$ is
$\dnormt{\cdot}$. 

Our basic assumptions correspond to the framework of
Algorithm~\ref{alg:adaptive}, which we summarize together with a few
technical conditions as follows:
\begin{setting}\label{setting}
  We consider the algorithm that selects points according to
  \eqr{update} based on convex $r_t$ that satisfy $r_t(x) \ge 0$ for
  $t \in \set{0, 1, 2, \dots }$, against a sequence of convex loss
  functions $f_t: \R^n \rightarrow \RS$. Further, letting $h\ztt =
  r\ztt + f\tt$ we assume $\dom h\ztt$ is non-empty. Recalling $x_t =
  \argmin_x h\zttm(x)$, we further assume $\partial f_t(x_t)$ is
  non-empty.
\end{setting}
The minor technical assumptions made here do not rule out any
practical applications. We can now introduce the theorems which will be our main focus. The first will typically be applied to FTRL-Centered algorithms such as \DA:
\begin{theorem}\label{thm:centered}\emph{\textbf{General FTRL Bound}}
  Consider Setting~\ref{setting}, and suppose the $r_t$ are chosen
  such that $h\ztt + f\ti = r\ztt + f_{1:t+1}$ is 1-strongly-convex
  \wrt some norm $\normt{\cdot}$.  Then, for any $\xs \in \R^n$ and
  for any $T>0$,
  \[
  \Regret_T(\xs) \leq r\ztTm(\xs) + \h \sum_{t=1}^T \dnormtm{g_t}^2.
  \]
\end{theorem}
Our second theorem handles proximal regularizers:
\begin{theorem}\label{thm:proximal} \emph{\textbf{FTRL-Proximal Bound}}
  Consider Setting~\ref{setting}, and further suppose the $r_t$ are
  chosen such that $h\ztt = r\ztt + f\tt$ is 1-strongly-convex \wrt
  some norm $\normt{\cdot}$, and further the $r_t$ are proximal, that
  is $x_t$ is a minimizer of $r_t$. Then, choosing any $g_t
  \in \partial f_t(x_t)$ on each round, for any $\xs \in \R^n$ and for
  any $T>0$,
  \[
  \Regret_T(\xs) \leq r\ztT(\xs) + \h \sum_{t=1}^T \dnormt{g_t}^2 .
  \]
\end{theorem}
We state these bounds in terms of strong convexity conditions on
$h\ztt$ in order to also cover the case where the $f_t$ are themselves
strongly convex.  In fact, if each $f_t$ is strongly convex, then we
can choose $r_t(x) = 0$ for all $t$, and Theorems~\ref{thm:centered}
and \ref{thm:proximal} produce \emph{identical} bounds (and
algorithms).\footnote{To see this, note in Theorem~\ref{thm:centered}
  the norm in $\dnormtm{g_t}$ is determined by the strong convexity of
  $f\tt$, and in Theorem~\ref{thm:proximal} the norm in $\dnormt{g_t}$
  is again determined by the strong convexity of $f\tt$.}  When it is
not known a priori whether the loss functions $f_t$ are strongly
convex, the $r_t$ can be chosen adaptively to add only as much strong
convexity as needed, following~\citet{bartlett07adaptive}.
On the other hand, when the $f_t$ are not strongly convex (e.g.,
linear), a sufficient condition for both theorems is choosing the
$r_t$ such that $r\ztt$ is 1-strongly-convex \wrt $\normt{\cdot}$.

It is worth emphasizing the ``off-by-one'' difference between
Theorems~\ref{thm:centered} and \ref{thm:proximal} in this case: we can choose
$r_t$ based on $g_t$, and when using proximal regularizers, this lets
us influence the norm we use to measure $g_t$ in the final bound
(namely the $\dnormt{g_t}^2$ term); this is not possible using
Theorem~\ref{thm:centered}, since we have $\dnormtm{g_t}^2$.  This makes
constructing AdaGrad-style adaptive learning rate algorithms for
FTRL-Proximal easier \citep{mcmahan10boundopt}, whereas with FTRL-Centered algorithms one must start with slightly more regularization.
We will see this in more detail in Section~\ref{sec:applications}.

Theorem~\ref{thm:centered} leads immediately to a bound for \DA
algorithms~\citep{nesterov09dualaveraging}, including the Regularized
Dual Averaging (RDA) algorithm of \citet{xiao09dualaveraging}, and its
AdaGrad variant~\citep{duchi11adaptivejournal} (in fact, this statement is
equivalent to~\citet[Prop. 2]{duchi11adaptivejournal} when we assume the
$f_t$ are not strongly convex).
As in these cases, Theorem~\ref{thm:centered} is usually applied to
FTRL-Centered algorithms where $x_1$ (often the origin) is a global minimizer of
$r\ztt$ for each $t$.  The theorem does not require this; however,
such a condition is usually necessary to bound $r\ztTm(\xs)$ and hence
$\Regret(\xs)$ in terms of $\norm{\xs}$.

Less general versions of these theorems often assume that each $r\ztt$
is $\alpha_t$-strongly-convex with respect to a fixed norm
$\norm{\cdot}$.  Our results include this as a special case, see Section~\ref{sec:applications} and Lemma~\ref{lem:basic} in particular.

\paragraph{Non-Adaptive Algorithms} These theorems can also be used to analyze non-adaptive algorithms.
If we choose $r_0(x)$ to be a fixed non-adaptive regularizer (perhaps
chosen with knowledge of $T$) that is 1-strongly convex \wrt
$\norm{\cdot}$, and all $r_t(x) = 0$ for $t \ge 1$, then we have
$\dnormt{x} = \dnorm{x}$ for all $t$, and so both theorems
provide the identical statement
\begin{equation}\label{eq:nonadaptive}
  \Regret(\xs) \le r_0(\xs) + \h \sum_{t=1}^T \dnorm{g_t}^2.
\end{equation}
This matches \citet[Theorem 2.11]{shwartz12online}, though we improve by a constant factor due to the use of the Strong FTRL Lemma.

\subsection{Incorporating a Feasible Set}\label{sec:feasible}
We have introduced the FTRL update as an unconstrained optimization
over $x \in \R^n$.  For many learning problems, where $x_t$ is a
vector of model parameters, this may be fine, but in other
applications we need to enforce constraints. These could correspond to
budget constraints, structural constraints like $\norm{x_t}_2 \le R$
or $\norm{x_t}_1 \le R_1$, a constraint that $x_t$ is a flow on a
graph, or that $x_t$ is a probability distribution. In all of these
cases, this amounts to the constraint that $x_t \in \X$
where $\X$ is a suitable convex feasible set. Further, for FTRL-Proximal algorithms a constraint like $\norm{x_t}_2 \le R$ is generally needed in order to bound $r\ztT(\xs)$; see Section~\ref{sec:ftrlprox}.

Such constraints can be addressed immediately in our setting by adding
the additional regularizer $\indX$ to $r_0$, based on the equivalence
\[
\argmin_{x \in \R^n} f\tt(x) + r\ztt(x) + \indX(x) 
 \quad = \quad
\argmin_{x \in \X} f\tt(x) + r\ztt(x). 
\]
Further, if $r\ztt$ satisfies the conditions of
Theorem~\ref{thm:centered}, then so does $r\ztt + \indX$. Similarly,
for Theorem~\ref{thm:proximal}, adding $\indX$ to $r_0$ will generally
still produce a scheme where $r_t$ has $x_t$ as a minimizer, and so
the theorem will still apply. We apply this technique to specific
algorithms in Section~\ref{sec:applications}.

Note that while the theorems still apply, the regret bounds change in
an important way, since $\indX(\xs)$ now appears in the regret bound:
that is, if Theorem~\ref{thm:centered} on functions $r_0, r_1, \dots,$
gives a bound $\Regret(\xs) \leq r\ztTm(\xs) + \h \sum_{t=1}^T
\dnormtm{g_t}^2$, then the version constrained to select from $\X$ by
adding $\indX$ to $r_0$ has regret bound
\[\Regret_T(\xs) \leq \indX(\xs) + r\ztTm(\xs) + \h \sum_{t=1}^T \dnormtm{g_t}^2.
\]
This bound is \emph{vacuous} for $\xs \not\in \X$, but identical to
the unconstrained bound for $\xs \in \X$. This makes sense: one can
show that any online algorithm constrained to select $x_t \in \X$
cannot in general hope to have sublinear regret against some $\xs
\not\in \X$.  Thus, if we believe some $\xs \not\in \X$ could perform
very well, incorporating the constraint $x_t \in \X$ is a significant
sacrifice that should only be made if external considerations really
require it.

\section{Application to Specific Algorithms and Settings}
\label{sec:applications}
Before proving these theorems, we apply them to a variety of specific
algorithms.  We will use the following lemma, which collects some
facts for the sequence of incremental regularizers
$r_t$.  These claims are immediate consequences of
the relevant definitions.
\begin{lemma}\label{lem:basic}
  Consider a sequence of $r_t$ as in Setting~\ref{setting}.
  Then, since $r_t(x) \geq 0$, we have $r\ztt(x) \ge r\zttm(x)$, and
  so $\rc\ztt(x) \le \rc\zttm(x)$, where $\rc\ztt$ is the convex-conjugate
  of $r\ztt$.  If each $r_t$ is $\sigma_t$-strongly convex \wrt a
  norm $\norm{\cdot}$ for $\sigma_t \geq 0$, then, $r\ztt$ is
  $\sigma\ztt$-strongly convex \wrt $\norm{\cdot}$, or equivalently,
  is $1$-strongly-convex \wrt $\normt{x} = \sqrt{\sigma\ztt}
  \norm{x}$, which has dual norm $\dnormt{x} = \frac{1}{\sqrt{\sigma\ztt}}\norm{x}$.
\end{lemma}
For reasons that will become clear, it is natural to define a learning rate schedule $\eta_t$ to be the inverse of the cumulative strong convexity,
\[ 
  \eta_t = \frac{1}{\sigma\ztt}.
\] 
In fact, in many cases it will be more natural to define the learning
rate schedule, and infer the sequence of $\sigma_t$,
\[
  \sigma_t = \frac{1}{\eta_t} - \frac{1}{\eta_{t-1}},
\]
with $\sigma_0 = \frac{1}{\eta_0}$.

For simplicity, in this section we assume the loss functions have
already been linearized, that is, $f_t(x) = g_t \cdot x$, unless
otherwise stated. Figure~\ref{fig:family} summarizes most of the FTRL
algorithms analyzed in this section.

\subsection{Constant Learning Rate \OGD}\label{sec:clrogd}
As a warm-up, we first consider a non-adaptive algorithm, unconstrained constant
learning rate \OGD, which selects
\begin{equation}\label{eq:gd}
x\ti = x_t - \eta g_t,
\end{equation}
where the parameter $\eta > 0$ is the learning rate.  Iterating this
update, we see $x\ti = -\eta g_{1:t}$.
There is a close connection between \OGD and FTRL, which
we will use to analyze this algorithm.  If we take FTRL with $r_0(x) =
\frac{1}{2\eta}\norm{x}_2^2$ and $r_t(x) = 0$ for $t \ge 1$, we have the
update
\begin{equation}\label{eq:gdftrl}
x\ti = \argmin_x g_{1:t} \cdot x + \frac{1}{2\eta}\norm{x}_2^2,
\end{equation}
which we can solve in closed form to see $x\ti = -\eta g_{1:t}$
as well.  Applying either Theorem \ref{thm:centered} or \ref{thm:proximal}
 (recall they are equivalent when the regularizer is
fixed) gives the bound of \eqr{nonadaptive}, in this case
\begin{equation}\label{eq:clrgdbound}
\Regret_T(\xs) \le \frac{1}{2\eta}\norm{\xs}_2^2 + \h \sum_{t=1}^T \eta \norm{g_t}^2_2,
\end{equation}
using Lemma~\ref{lem:basic} for $\dnormt{x} = \sqrt{\eta}\norm{x}_2$.
Suppose we are concerned with $\xs$ where $\norm{\xs}_2 \le R$, the
$g_t$ satisfy $\norm{g_t}_2 \le G$, and we want to minimize regret
after $T'$ rounds.  Then, choosing $\eta = \frac{R}{G\sqrt{T'}}$
minimizes \eqr{clrgdbound} when $T = T'$, and we have
\[
\Regret_T(\xs) \le \frac{R G}{2}\sqrt{T'} + \frac{R G }{2}\frac{T}{\sqrt{T'}},
\]
or $\Regret(\xs) \le R G\sqrt{T}$ when $T = T'$. However, this bound
is only $\BO(\sqrt{T})$ when $T = \BO(T')$. For $T \ll T'$, or $T \gg
T'$ the bound is no longer interesting, and in fact the algorithm will
likely perform poorly. This deficiency can be addressed via the
``doubling trick'', where we double $T'$ and restart the algorithm
each time $T$ grows larger than $T'$ (c.f.,
\citet[2.3.1]{shwartz12online}). However, adaptively choosing the
learning rate without restarting will allow us to achieve better
bounds than the doubling trick (by a constant factor) with a more practically
useful algorithm. We do this in Sections \ref{sec:da} and \ref{sec:ftrlprox} below.

\paragraph{Constant Learning Rate \OGD with a Feasible Set}
Above we assumed $\norm{\xs}_2 \le R$, but there is no a priori bound
on the magnitude of the $x_t$ selected by the algorithm. Following the
approach of Section~\ref{sec:feasible}, we can incorporate a feasible
set by taking $r_0(x) = \frac{1}{2\eta}\norm{x}_2^2 + \indX(x),$ so
the update becomes
\begin{align}\label{eq:constrainedgd}
x\ti &= \quad \argmin_{x \in \R^n} g_{1:t} \cdot x + \frac{1}{2\eta}\norm{x}_2^2 + \indX(x) 
\quad = \quad  \argmin_{x \in \X} g_{1:t} \cdot x + \frac{1}{2\eta}\norm{x}_2^2.
\end{align}
This update is in fact equivalent to the two-step update where we
first solve the unconstrained problem and then project onto the
feasible set, namely
\begin{align*}
u\ti &= \argmin_{x \in \R^n} g_{1:t} \cdot x + \frac{1}{2\eta}\norm{x}^2\\
x\ti &= \projX(u\ti) \qq{where}  \projX(u) \equiv \argmin_{x \in \X} \norm{x - u}_2.
\end{align*}
Many FTRL algorithms on feasible sets can in this way be interpreted
as lazy-projection algorithms, where we find (or maintain) the
solution to the unconstrained problem, and then project onto the
feasible set when needed.

Theorem~\ref{thm:centered} can be used to analyze the constrained
algorithm of \eqr{constrainedgd} in exactly the same way we analyzed
\eqr{gdftrl}: adding $\indX$ does not change the strong convexity of
the $\norm{x}_2^2$ terms in the regularizer, and so the only
difference is in the $r\ztT(\xs)$ term. Instead of \eqr{clrgdbound},
we have
\begin{equation*}
\forall \xs \in \X,\ \Regret_T(\xs) \le  \frac{1}{2\eta}\norm{\xs}_2^2 + \h \sum_{t=1}^T \eta \norm{g_t}^2_2,
\end{equation*}
where we have chosen to use the explicit quantification $\xs \in \X$
rather than the equivalent choice of including $\indX(\xs)$ on the
right-hand side.

Interestingly, the update of \eqr{constrainedgd} is no longer
equivalent to the standard projected \OGD update $x\ti = \projX(x_t -
\eta g_t)$; this issue is discussed in the context of more general \MD
updates in Appendix~\ref{sec:lazyvgreedy}. We will be able to analyze
this algorithm using techniques from Section~\ref{sec:md}.

\subsection{\DA}\label{sec:da}
Dual Averaging is an adaptive FTRL-Centered algorithm with linearized
loss functions; the adaptivity allows us to prove regret bounds that
are $\BO(\sqrt{T})$ for all $T$.  We choose $r_t(x) =
\frac{\sigma_t}{2}\norm{x}_2^2$ for constants $\sigma_t \ge 0$, so
$r_{0:t}$ is 1-strongly-convex \wrt the norm $\normt{x} =
\sqrt{\sigma_{0:t}}\norm{x}_2$, which has dual norm $\dnormt{x} =
\frac{1}{\sqrt{\sigma_{0:t}}}\norm{x}_2 = \sqrt{\eta_t}\norm{x}_2$,
using Lemma~\ref{lem:basic}. Plugging into Theorem~\ref{thm:centered}
then gives
\[
\forall T, \ \Regret_T(\xs) \leq
 \frac{1}{2 \eta_{T-1}} \norm{\xs}_2^2 +
 \h \sum_{t=1}^T \eta_{t-1} \norm{g_t}^2_2.
\]
Suppose we know $\norm{g_t}_2 \leq G$, and we consider $\xs$ where
$\norm{\xs}_2 \leq R$.  Then, with the choice $\eta_{t} =
\frac{R}{\sqrt{2}G\sqrt{t+1}}$, using the inequality $\sum_{t=1}^T
\frac{1}{\sqrt{t}} \leq 2 \sqrt{T}$, we arrive at
\begin{align} \label{eq:rdabound}
\forall T, \ \Regret_T(\xs)
\le  \frac{\sqrt{2}}{2}\left(R + \frac{\norm{\xs}_2^2}{R}\right) G \sqrt{T}.
\end{align}
When in fact $\norm{\xs} \leq R$, we have $\Regret \le \sqrt{2} R G
\sqrt{T}$, but the bound of \eqr{rdabound} is valid (and meaningful)
for arbitrary $\xs \in \R^n$. Observe that on a particular round $T$,
this bound is a factor $\sqrt{2}$ worse than the bound of $R G
\sqrt{T}$ shown in Section~\ref{sec:clrogd} when the learning rate is
tuned for exactly round $T$; this is the (small) price we pay for a bound that holds uniformly for all $T$.

As in the previous example, \DA can also be restricted to select from
a feasible set $\X$ by including $\indX$ in
$r_0$. Additional non-smooth regularization can also be applied by
adding the appropriate terms to $r_0$ (or any of the $r_t$); for
example, we can add an $L_1$ and $L_2$ penalty by adding the terms
$\lambda_1\norm{x}_1 + \lambda_2\norm{x}_2^2$.  When in addition the
$f_t$ are linearized, this produces the \RDA algorithm of
\citet{xiao09dualaveraging}.  Note that our result of $\sqrt{2} R G
\sqrt{T}$ improves on the bound of $2 R G \sqrt{T}$ achieved by
\citet[Cor.~2(a)]{xiao09dualaveraging}.
We consider the case of such additional regularization terms in more
detail in Section~\ref{sec:composite}.

\subsection{FTRL-Proximal}\label{sec:ftrlprox}
Suppose $\X \subseteq \{x \mid \norm{x}_2 \le R\}$, and we choose
$r_0(x) = \indX(x)$ and for $t>1$, $r_t(x) = \frac{\sigma_t}{2}\norm{x
  - x_t}_2^2$. It is worth emphasizing that unlike in the previous
examples, for FTRL-Proximal the inclusion of the feasible set $\X$ is
essential to proving regret bounds. With this constraint we have
$r\ztt(\xs) \leq \frac{\sigma\tt}{2}(2R)^2$ for any $\xs \in \X$,
since each $x_t \in \X$.  Without forcing
$x_t \in \X$, however, the terms $\norm{\xs - x_t}_2^2$ in
$r\ztt(\xs)$ cannot be usefully bounded.

With these choices, $r\ztt$ is 1-strongly-convex \wrt the norm
$\normt{x} = \sqrt{\sigma\tt}\norm{x}_2$, which has dual norm
$\dnormt{x} = \frac{1}{\sqrt{\sigma\tt}}\norm{x}_2$. Thus, applying
Theorem~\ref{thm:proximal}, we have
\begin{equation}\label{eq:proxquadbound}
\forall \xs \in \X,  \quad
\Regret(\xs)
\le   \frac{1}{2 \eta_T}(2R)^2 + \h \sum_{t=1}^T \eta_t \norm{g_t}^2,
\end{equation}
where again $\eta_t = \frac{1}{\sigma\tt}$.  Choosing $\eta_{t} =
\frac{\sqrt{2}R}{G\sqrt{t}}$ and assuming $\norm{\xs} \le R$ and $\norm{g_t}_2 \le G$, \begin{equation}\label{eq:proxquadG}
\Regret(\xs)
 \le 2\sqrt{2} R G \sqrt{T}.
\end{equation}
Note that we are a factor of 2 worse than the corresponding bound for
\DA.  However, this is essentially an artifact of loosely bounding
$\norm{\xs - x_t}_2^2$ by $(2R)^2$, whereas for \DA we can bound
$\norm{\xs - 0}_2^2$ with $R^2$.  In practice one would hope $x_t$ is
closer to $\xs$ than $0$, and so it is reasonable to believe that the
FTRL-Proximal bound will actually be tighter post-hoc in many
cases.
Empirical evidence also suggests FTRL-Proximal can work better in
practice \citep{mcmahan10equiv}.

\subsection{FTRL-Proximal with Diagonal Matrix Learning Rates}
\label{sec:adagradprox}
We now consider an AdaGrad FTRL-Proximal algorithm which is adaptive
to the observed sequence of gradients $g_t$, improving on the previous
result.  For simplicity, first consider a one-dimensional problem.
Let $r_0 = \indX$ with $\X = [-R, R]$, and fix a learning rate
schedule for FTRL-Proximal where
\[
 \eta_t = \frac{\sqrt{2}R}{\sqrt{\sum_{s=1}^t g_s^2}}
\]
for use in~\eqr{proxquadbound}.  This gives
\begin{align}
\Regret(\xs)
  \le 2\sqrt{2}R\sqrt{\sum_{t=1}^T g_t^2}, \label{eq:percoord}
\end{align}
where we have used the following lemma, which generalizes
$\sum_{t=1}^T 1/\sqrt{t} \leq 2\sqrt{T}$:
\begin{lemma}\label{lem:sum}
For any non-negative real numbers $a_1, a_2, \ldots, a_n$,
\[
\sum_{i=1}^n \frac { a_i } { \sqrt { \sum_{j=1}^i a_j } } \le 2 \sqrt
    { \sum_{i=1}^n a_i } \mbox { .}
\]
\end{lemma}
For a proof see \citet{auer00adaptive} or
\citet[Lemma~1]{streeter10conditioning}.
The bound of \eqr{percoord}
gives us a fully adaptive version of \eqr{proxquadG}: not only do we
not need to know $T$ in advance, we also do not need to know a bound
on the norms of the gradients $G$.  Rather, the bound is fully
adaptive and we see, for example, that the bound only depends on
rounds $t$ where the gradient is nonzero (as one would hope).  We do,
however, require that $R$ is chosen in advance; for algorithms that
avoid this, see
\citet{streeter12unconstrained,orabona13dimfree,mcmahan13minimax}, and
\citet{mcmahan14minimax}.

\newcommand{\Rinf}{R_\infty}
To arrive at an AdaGrad-style algorithm for $n$-dimensions we
need only apply the above technique on a per-coordinate basis, namely using learning rate
\[
 \eta_{t,i} = \frac{\sqrt{2}\Rinf}{\sqrt{\sum_{s=1}^t g_{s,i}^2}}
\]
for coordinate $i$, where we assume $\X \subseteq [-\Rinf, \Rinf]^n$.
\citet{streeter10conditioning} take the per-coordinate approach
directly; the more general approach here allows us to handle arbitrary
feasible sets and $L_1$ or other non-smooth regularization.

We take $r_0 = \indX$, and for $t \ge 1$ define $r_t(x) = \h
\norm{Q_t^\h(x -x_t)}_2^2$ where $Q_t = \diag\big(\sigma_{t,i})$, the
diagonal matrix with entries $\sigma_{t,i} = \eta_{t,i}\inv -
\eta_{t-1,i}\inv$. This $Q_t$ is positive semi-definite, and for any
such $Q_t$, we have that $r\ztt$ is 1-strongly-convex \wrt the norm
$\normt{x} = \norm{(Q\tt)^\h x}_2$, which has dual norm $\dnormt{g} =
\norm{(Q\tt)^{-\h} g}_2$. Then, plugging into
Theorem~\ref{thm:proximal} gives
\[
\Regret(\xs) \le r\ztT(\xs) + \h \sum_{t=1}^T \norm{(Q\tt)^{-\h} g_t}_2.
\]
which improves on
\citet[Theorem~2]{mcmahan10boundopt} by a constant factor.

Essentially, this bound amounts to summing \eqr{percoord} across all
$n$ dimensions; \citet[Cor. 9]{mcmahan10boundopt} show this bound is
at least as good (and often better) than that of \eqr{proxquadG}.
Full matrix learning rates can be derived using a matrix
generalization of Lemma~\ref{lem:sum}, e.g., \citet[Lemma
10]{duchi11adaptivejournal}; however, since this requires $\BO(n^2)$
space and potentially $\BO(n^2)$ time per round, in practice these
algorithms are often less useful than the diagonal varieties.

It is perhaps not immediately clear that the diagonal FTRL-Proximal
algorithm is easy and efficient to implement.  In fact, however,
taking the linear approximation to $f_t$, one can see $h_{1:t}(x) =
g_{1:t}\cdot x + r_{1:t}(x)$ is itself just a quadratic which can be
represented using two length $n$ vectors, one to maintain the linear
terms ($g_{1:t}$ plus adjustment terms) and one to maintain
$\sum_{s=1}^t g_{s,i}^2$, from which the diagonal entries of $Q\tt$
can be constructed.  That is, the update simplifies to
\begin{align*}
x\ti 
  &=\argmin_{x \in \X}\ (g\tt - a\tt)\cdot x + \sum_{i=1}^n \frac{1}{2\eta_{t,i}} x_i^2
 \qq{where} \text{$a_t = \sigma_t x_t$}.
\end{align*}
This update can be solved in closed-form on a per-coordinate basis
when $\X = [-\Rinf, \Rinf]^n$. For a general feasible set, it is equivalent to
a lazy-projection algorithm that first solves for the unconstrained
solution and then projects it onto $\X$ using
norm $\norm{(Q\tt)^{\h} \cdot}$ (see
\citet[Eq.~7]{mcmahan10boundopt}). Pseudo-code which also
incorporates $L_1$ and $L_2$ regularization is given in \citet{mcmahan13adclick}.

\subsection{AdaGrad Dual Averaging}
\label{sec:adagraddual}
Similar ideas can be applied to \DA (where we center each $r_t$ at
$x_1$), but one must use some care due to the ``off-by-one''
difference in the bounds.  For example, for the diagonal algorithm, it
is necessary to choose per-coordinate learning rates
\[
\eta_t \approx \frac{R}{\sqrt{G^2 + \sum_{s=1}^t g_s^2}},
\]
where $\abs{g_t} \leq G$.  Thus, we arrive at an algorithm that is
almost (but not quite) fully adaptive in the gradients, since a modest
dependence on the initial guess $G$ of the maximum per-coordinate
gradient remains in the bound.  This offset appears, for example, as
the $\delta I$ terms added to the learning rate matrix $H_t$ in
Figure~1 of \citet{duchi11adaptivejournal}. We will explore this issue
in more detail in the following example.

\newcommand{\Ginf}{G_\infty} 
\newcommand{\myfbox}[1]{\fbox{#1}}
\newcommand{\qwe}[1]{\qquad\text{#1}}
\newcommand{\secT}{\rule{0pt}{5.2ex}}
\begin{figure}[p] %
  \setlength{\fboxsep}{5pt}
  \myfbox{%
    \begin{minipage}{\textwidth}
      \vspace{-0.2in}
      \begin{small}
\begin{align*}
\shortintertext{\textbf{Non-Adaptive FTRL Algorithms}  (fixed regularizer $r_0$, with $r_t(x) = 0$ for $t \ge 1$)}
\shortintertext{Constant Learning Rate Unprojected \OGD}
x\ti &= x_t - \eta g_t \\
&= \argmin_x g\tt \cdot x_t + \frac{1}{2\eta}\norm{x}_2^2 \\
&= -\eta g_{1:t}
\shortintertext{Follow-The-Leader where the $f_t$ are 1-strongly-convex w.r.t. $\norm{\cdot}$}
x\ti &= \argmin_x f\tt(x) 
\shortintertext{\OGD for strongly-convex functions}
x\ti &%
= \argmin_x g\tt \cdot x + \h \sum_{s=1}^t \norm{x - x_s}^2
\qwe{where $g_t \in \partial f_t(x_t)$}\\
&= x_t - \eta_t g_t \qwe{where $\eta_t = \frac{1}{t}$}
\shortintertext{\secT\textbf{Adaptive FTRL-Centered Algorithms} ($r_t$ chosen adaptively and minimized at $x_1$)}
\shortintertext{Unconstrained Dual Averaging (adaptive to $t$)}
x\ti &= \argmin_x g\tt\cdot x + \frac{1}{2 \eta_t} \norm{x}_2^2 \qwe{where $\eta_{t} =
\frac{R}{\sqrt{2}G\sqrt{t+1}}$} \\
 &= -\eta_t g_{1:t}
\shortintertext{FTRL with the entropic regularizer over the probability simplex $\Delta$ (adaptive to $g_t$)}
x\ti &= \argmin_{x\in\Delta} g\tt \cdot x + \frac{1}{2\eta_t}\sum_{i=1}^n x_i \log x_i 
\qwe{where $\eta_t = \frac{\sqrt{\log n}}{\sqrt{\Ginf^2 + \sum_{s=1}^t \norm{g_s}_\infty^2}}$, or}\\
&x_{t+1, i} = \frac{\exp(-\eta_t g_{1:t,i})}{\sum_{i=1}^n  \exp(-\eta_t g_{1:t,i})} \qwe {in closed form}
\shortintertext{\secT\textbf{Adaptive FTRL-Proximal Algorithms}  ($r_t$ chosen adaptively and minimized at $x_t$)}
\shortintertext{FTRL-Proximal (adaptive to $t$)
with $\sigma_s = \eta_s\inv - \eta_{s-1}\inv$
}
x\ti 
  &= \argmin_{x \in \X} g\tt\cdot x + \sum_{s=1}^t\frac{\sigma_s}{2}\norm{x - x_s}_2^2       
  \qwe{where $\eta_{t} = 
    \frac{\sqrt{2}R}{G\sqrt{t}}$ }
\shortintertext{AdaGrad FTRL-Proximal (adaptive to $g_t$) with
$\sigma_{s,i} = \eta_{s,i}\inv - \eta_{s-1,i}\inv$.}
x\ti 
  &= \argmin_{x \in \X} g\tt\cdot x + \sum_{s=1}^t
 \h \Big\|\diag\big(\sigma_{s,i}^\h\big) (x -x_s)\Big\|_2^2
  \qwe{where 
$\eta_{t,i} = \frac{\sqrt{2}{R}}{\sqrt{\sum_{s=1}^t g_{s,i}^2}}$}
\end{align*}
\end{small}
\end{minipage}%
}
\caption{Example updates for algorithms in different branches of the FTRL family.}\label{fig:family}
\end{figure}

\subsection{Adaptive Dual Averaging with the Entropic Regularizer}
We consider problems where the algorithm selects a probability distribution (e.g., in order to sample an action from a discrete set of $n$ choices), that is $x_t \in \Delta_n$ with 
\[\Delta_n = \left\{ x \ \big\rvert\ \sum\nolimits_{i=1}^n x_i = 1\ \text{and}\ x_i \ge 0\right\}.
\]
We assume gradients are bounded so that $\norm{g_t}_\infty \le \Ginf$,
which is natural for example if each action has a cost in the range
$[-\Ginf, \Ginf]$, so $g_t \cdot x$ gives the expected cost of
choosing an action from the distribution $x$.  This is the classic
problem of prediction from expert advice
\citep{vovk90aggregating,littlestone94wm,hedge,cesabianchi06plg}.

The previously introduced algorithms can be applied by enforcing the
constraint $x \in \Delta_n$ by adding $\ind_{\Delta_n}$ to $r_0$, but
to instantiate their bounds we can only bound $\norm{g_t}_2$ by
$\sqrt{n}\Ginf$ in this case, leading to bounds like $\BO(\Ginf\sqrt{n
  T})$. By using a more
appropriate regularizer, we can reduce the dependence on
the dimension from $\sqrt{n}$ to $\sqrt{\log n}$. In particular, we use the
entropic regularizer,
\[
h(x) = I_\Delta(x) + \log n + \sum_{i=1}^n x_i \log x_i, %
\]
from which we define the following adaptive regularization schedule:
\[
r\ztt(x) = \frac{1}{\eta_t} h(x)
\qqwhere
\eta_t = \frac{\sqrt{\log n}}{\sqrt{\Ginf^2 + \sum_{s=1}^t \norm{g_s}_\infty^2}}
\]
for $t \ge 0$.  Note that as in AdaGrad Dual Averaging, we make the
learning rate schedule $\eta_t$ a function of the observed $g_t$. The
function $h$ (and hence each $r\ztt)$ is minimized by the uniform
distribution $x_1 = (1/n, \dots, 1/n)$ where $h(x) = 0$, and so these
regularizers are centered at $x_1$. Note also that $h$ is maximized at
the corners of $\Delta_n$ (e.g., $x = (1, 0, \dots, 0)$) where it has
value $\log n$.

The entropic regularizer $h$ is 1-strongly-convex with respect to the $L_1$ norm over the probability simplex $\X$ (e.g., \citet[Ex 2.5]{shwartz12online}), and it follows that $r\ztt$ is $1$-strongly convex with respect to the norm
$\normt{x} = \frac{1}{\sqrt{\eta_t}} \norm{x}_1$, and $\dnormt{g}^2 =
\eta_t \norm{g}_\infty^2.$  Then, applying Theorem~\ref{thm:centered}, we have
\begin{align*}
\Regret(\xs)
 &\le r\ztTm(\xs) + \h \sum_{t=1}^T \dnormtm{g_t}^2 \\
 &\le \frac{\log n}{\eta_{T-1}} + \h \sum_{t=1}^T \eta_{t-1} \norm{g_t}_\infty^2 \\
 &\le \frac{\log n}{\eta_{T-1}} + \frac{\sqrt{\log n}}{2} \sum_{t=1}^T \frac{ \norm{g_t}_\infty^2}{\sqrt{\sum_{s=1}^t \norm{g_s}_\infty^2}} 
    && \text{since $\forall t,\ \norm{g_t}_\infty \le \Ginf$} \\
 &\le 2 \sqrt{\left(\Ginf^2 + \sum_{t=1}^{T-1} \norm{g_t}_\infty^2\right)\log n }
    && \text{Lemma~\ref{lem:sum} and $\norm{g_T}_\infty \le \Ginf$} \\
 & \le 2 \Ginf \sqrt{T \log n}.
\end{align*}
The last line gives an adaptive ($\forall T$) version of
\citet[Cor. 2.14 and Cor 2.16]{shwartz12online}, but the version of
the bound in terms of $\norm{g_t}_\infty$ may be much tighter if there are
many rounds where the maximum magnitude cost is much lass than
$\Ginf$.  For similar adaptive algorithms, see \citet[Thm
2.3]{stoltz05thesis} and \citet[Thm 1.4, Eq. (1.22)]{stoltz11thesis}.

\subsection{Strongly Convex Functions}
Suppose each loss function $f_t$ is 1-strongly-convex \wrt a norm
$\norm{\cdot}$, and let $r_t(x) = 0$ for all $t$ (that is, we use
the Follow-The-Leader (FTL) algorithm).  Define $\normt{x} = \sqrt{t}
\norm{x}$, and observe $h\ztt(x)$ is 1-strongly-convex \wrt
$\normt{\cdot}$ (by Lemma~\ref{lem:basic}).  Then, applying either
Theorem \ref{thm:centered} or \ref{thm:proximal} (recalling they coincide when all $r_t(x) = 0$),
\begin{align*}
\Regret(\xs)
\le \h \sum_{t=1}^T \dnormt{g_t}^2
= \h \sum_{t=1}^T \frac{1}{t} \norm{g_t}^2
&\le \frac{G^2}{2}(1 + \log{T}),
\end{align*}
where we have used the inequality $\sum_{t=1}^T 1/t \leq 1 +
\log T$ and assumed $\norm{g_t} \leq G$.  This recovers, e.g.,
\citet[Cor. 1]{kakade08duality} for the the exact FTL algorithm. This
algorithm requires optimizing over $f_{1:t}$ exactly, which may be
computationally prohibitive.

For a 1-strongly-convex $f_t$ with $g_t \in \partial
f_t(x_t)$ we have by definition
\[
f_t(x) \ge \underbrace{f_t(x_t) + g_t \cdot (x - x_t)
  + \h \norm{x - x_t}^2}_{=\hf_t}.
\]
Thus, we can define a $\hf_t$ equal to the right-hand-side of the
above inequality, so $\hf_t(x) \leq f_t(x)$ and $\hf_t(x_t) =
f_t(x_t)$.  The $\hf_t$ are also 1-strongly-convex \wrt
$\norm{\cdot}$, and so running FTL on these functions produces an
identical regret bound. Theorem~\ref{thm:mdequiv} will show that the update $x\ti = \argmin_x \hf\tt(x)$ is equivalent to the \OGD update
\[
     x\ti = x_t - \frac{1}{t} g_t,
\]
showing this update is essentially the \OGD
algorithm for strongly convex functions given
by~\citet{hazan07logarithmic}.\footnote{Again, the constraint to select
  from a fixed feasible set $\X$ can be added easily in either case;
  however, the natural way to add the constraint to the FTRL
  expression produces a ``lazy-projection'' algorithm, whereas adding the constraint to the \OGD update produces a ``greedy-projection'' algorithm. This issue is discussed in some depth in Appendix~\ref{sec:lazyvgreedy}.}

\section{A General Analysis Technique}\label{sec:genproof}
In this section, we prove Theorems \ref{thm:centered} and
\ref{thm:proximal}; the analysis techniques developed will also be
used in subsequent sections to analyze composite objectives and \MD
algorithms.

\subsection{Inductive Lemmas}\label{sec:inductive}

In this section we prove the following lemma that lets us analyze arbitrary
FTRL-style algorithms:
\begin{lemma}[Strong FTRL Lemma]\label{lem:strong_ftrl}
  Let $f_t$ be a sequence of arbitrary (possibly non-convex) loss
  functions, and let $r_t$ be arbitrary non-negative regularization
  functions, such that $x_{t+1} = \argmin_x h\ztt(x)$ is well defined, where $h\ztt(x) \equiv  f\tt(x) + r\ztt(x)$.  Then, the algorithm that
  selects these $x_t$ achieves
  \vspace{-0.12in}
  \begin{equation}\label{eq:strongftrl}
    \Regret(\xs)
       \leq r\ztT(\xs) + \sum_{t=1}^T h\ztt(x_t)
           - h\ztt(x\ti) - r_t(x_t).
  \end{equation}
\end{lemma}
This lemma can be viewed as a stronger form of the more well-known
standard FTRL Lemma (see \cite{kalai03ftpl,hazan08extract},
\citet[Lemma 1]{hazan10survey}, \citet[Lemma 3]{mcmahan10boundopt},
and \citet[Lemma 2.3]{shwartz12online}). The strong version has three
main advantages over the standard version: 1) it is essentially tight,
which improves the final bounds by a constant factor, 2) it can be
used to analyze adaptive FTRL-Centered algorithms in addition to
FTRL-Proximal, and 3) it relates directly to the primal-dual style of
analysis. For completeness, in Appendix~\ref{sec:stdftrl} we present
the standard version of the lemma, along with the proof of a bound
analogous to Theorem~\ref{thm:proximal} (but weaker by a constant
factor).

The Strong FTRL Lemma bounds regret by the sum of two factors:
\begin{itemize}
\item \textbf{Stability} The terms in the sum over $t$ measure how
  much better $x\ti$ is for the cumulative objective function $h\ztt$
  than the point actually selected, $x_t$: namely $h\ztt(x_t) -
  h\ztt(x\ti)$. These per-round terms can be seen as measuring the
  stability of the algorithm, an online analog to the role of
  stability in the stochastic setting
  \citep{bousquet02stability,rakhlin05stability,shwartz10stability}.

\item \textbf{Regularization} The term $r\ztT(\xs)$ quantifies how much
  regularization we have added, measured at the comparator point
  $\xs$. This captures the intuitive fact that if we could center our
  regularization at $\xs$ it should not increase regret.
\end{itemize}
Adding strongly convex regularizers will increase stability (and hence
decrease the cost of the stability terms), at the expense of paying a
larger regularization penalty $r\ztT(\xs)$. At the heart
of the adaptive algorithms we study is the ability to dynamically
balance these two competing goals.

The following corollary relates the above statement to the primal-dual
style of analysis:
\begin{corollary}\label{cor:strongconj}
  Consider the same conditions as Lemma~\ref{lem:strong_ftrl}, and further
  suppose the loss functions are linear, $f_t(x) = g_t \cdot x_t$.
  Then,
  \begin{equation}\label{eq:strongconj}
    h\ztt(x_t) - h\ztt(x\ti) - r_t(x_t)
   =
   \rc\ztt(-g\tt) -\rc\zttm(-g\ttm) + g_t \cdot x_t,
  \end{equation}
  which implies
  \begin{equation*}
    \Regret(\xs) \leq r\ztT(\xs) + \sum_{t=1}^T
    \rc\ztt(-g\tt) -\rc\zttm(-g\ttm) + g_t \cdot x_t.
  \end{equation*}
\end{corollary}
We make a few remarks before proving these results at the end of this
section.
Corollary~\ref{cor:strongconj} can easily be proved directly using the
Fenchel-Young inequality. Our statement directly matches the first claim of
\citet[Lemma~1]{orabona13dimfree}, and in the non-adaptive case
re-arrangement shows equivalence to \citet[Lemma
1]{shwartz07thesis} and \citet[Lemma 2.20]{shwartz12online};
see also \citet[Corollary~4]{kakade12regularization}.
\citet[Thm. 1]{mcmahan14minimax} give a closely related duality result
for regret and reward, and discuss several interpretations for this
result, including the potential function view, the connection to
Bregman divergences, and an interpretation of $\rc$ as a benchmark
target for reward.

Note, however, that Lemma~\ref{lem:strong_ftrl} is strictly stronger
than Corollary~\ref{cor:strongconj}: it applies to non-convex $f_t$
and $r_t$.  Further, even for convex $f_t$, it can be more
useful: for example, we can directly analyze strongly convex $f_t$
with all $r_t(x) = 0$ using the first statement.
Lemma~\ref{lem:strong_ftrl} is also arguably simpler, in that it does
not require the introduction of convexity or the Fenchel conjugate.
We now prove the Strong FTRL Lemma:
\begin{proof}[Proof of Lemma~\ref{lem:strong_ftrl}]
  First, we bound a quantity that is essentially our regret if we had
  used the FTL algorithm against the functions $h_1, \dots h_T$ (for
  convenience, we include a $-h_0(\xs)$ term as well):
\begin{align*}
   \sum_{t=1}^T h_t(&x_t) - h\ztT (\xs)\\
   &= \sum_{t=1}^T (h\ztt(x_t) - h\zttm(x_{t})) -h\ztT(\xs) \\
   &\leq \sum_{t=1}^T (h\ztt(x_t) - h\zttm(x_{t})) -h\ztT(x_{T+1})
        && \text{Since $x_{T+1}$ minimizes $h\ztT$}\\
   &\leq \sum_{t=1}^T (h\ztt(x_t) - h\ztt(x_{t+1})),
\end{align*}
where the last line follows by simply re-indexing the $-h\ztt$ terms
and dropping the the non-positive term $-h_0(x_1) = -r_0(x_1) \le 0$.
Expanding the definition of $h$ on the left-hand-side of the above
inequality gives
\[
  \sum_{t=1}^T (f_t(x_t) + r_t(x_t)) - f\tT(\xs) - r\ztT(\xs)
     \leq \sum_{t=1}^T (h\ztt(x_t) - h\ztt(x_{t+1})).
\]
Re-arranging the inequality proves the lemma.
\end{proof}
We remark it is possible to make Lemma~\ref{lem:strong_ftrl} an
\emph{equality} if we include the non-positive term $h\tT(x_{T+1}) -
h\tT(\xs)$ on the RHS, since we can assume $r_0(x_1) = 0$ without loss
of generality.  Further, if one is actually interested in the
performance of the Follow-The-Leader (FTL) algorithm against the $h_t$
(e.g., if all the $r_t$ are uniformly zero), then choosing $\xs =
x_{T+1}$ is natural. 

\begin{proof}[Proof of Corollary~\ref{cor:strongconj}]
  Using the definition of the Fenchel conjugate and of $x\ti$,
  \begin{equation}\label{eq:rch}
    \rc\ztt(-g\tt)
    = \max_x\ -g\tt \cdot x - r\ztt(x)
    = - \big(\min_x\ g\tt \cdot x + r\ztt(x) \big)
    = - h\ztt(x_{t+1}).
  \end{equation}
  Now, observe that
  \begin{align*}
    h\ztt(x_t) - r_t(x_t)
    &= g\tt\cdot x_t + r\ztt(x_t) - r_t(x_t) \\
    &= g\ttm\cdot x_t + r\zttm(x_t) + g_t \cdot x_t \\
    &= h\zttm(x_t) + g_t \cdot x_t \\
    &= -\rc\zttm(-g\ttm) + g_t \cdot x_t,
  \end{align*}
  where the last line uses \eqr{rch} with $t \rightarrow t-1$.
  Combining this with \eqr{rch} again ($-h\ztt(x_{t+1}) = \rc\ztt(-g\tt)$) proves
  \eqr{strongconj}.
\end{proof}

\subsection{Tools from Convex Analysis}\label{sec:convex}
Here we highlight a few key tools from convex analysis that will be
used to bound the per-round stability terms that appear in the
Strong FTRL Lemma.  For more background on convex analysis,
see~\citet{rockafellar} and \citet{shwartz07thesis,shwartz12online}.  The next
result generalizes arguments found in earlier proofs for FTRL
algorithms:
\begin{lemma}\label{lem:smoothchange}
  Let $\ha: \R^n \rightarrow \RS$ be a convex function such that $x_1
  = \argmin_x \ha(x)$ exists.  Let $\psi$ be a convex function such
  that $\hb(x) = \ha(x) + \psi(x)$ is strongly convex \wrt norm
  $\norm{\cdot}$.  Let $x_2 = \argmin_x \hb(x)$. Then, for
  any $b \in \partial \psi(x_1)$, we have
  \begin{equation}\label{eq:xchange}
  \norm{x_1 - x_2} \leq \dnorm{b},
  \end{equation}
  and for any $x'$,
  \[
  \hb(x_1) - \hb(x') \le \h \dnorm{b}^2.
  \]
\end{lemma}
\noindent
We defer the proofs of the results in this section to
Appendix~\ref{sec:convexproofs}.  When $\ha$ and $\psi$ are quadratics
(with $\psi$ possibly linear) and the norm is the corresponding $L_2$ norm,
both statements in the above lemma hold with equality.  For the
analysis of composite updates (Section~\ref{sec:composite}), it will
be useful to split the change $\psi$ in the objective function $\phi$
into two components:
\begin{corollary}\label{cor:smoothchange2}
  Let $\ha: \R^n \rightarrow \RS$ be a convex function such that $x_1
  = \argmin_x \ha(x)$ exists.  Let $\psi$ and $\Psi$ be convex
  functions such that $\hb(x) = \ha(x) + \psi(x) + \Psi(x)$ is
  strongly convex \wrt norm $\norm{\cdot}$.  Let $x_2 = \argmin_x
  \hb(x)$.  Then, for any $b \in \partial \psi(x_1)$ and any $x'$,
  \[
  \hb(x_1) - \hb(x') \le \h \dnorm{b}^2 + \Psi(x_1) - \Psi(x_2).
  \]
\end{corollary}

The concept of strong smoothness plays a key role in the proof of the
above lemma, and can also be used directly in the application of
Corollary~\ref{cor:strongconj}.  A function $\psi$ is
$\sigma$\defn{-strongly-smooth}
with respect to a norm $\norm{\cdot}$
if it is differentiable and for all $x,y$ we have
\begin{equation}\label{eq:strongsmooth}
\psi(y)
 \leq \psi(x) + \grad \psi(x)\cdot (y - x) + \tfrac{\sigma}{2} \norm{y - x}^2.
\end{equation}
There is a fundamental duality between strongly convex and strongly
smooth functions:
\begin{lemma}\label{lem:strongdual}
  Let $\psi$ be closed and convex.  Then $\psi$ is
  $\sigma$-strongly convex with respect to the norm $\norm{\cdot}$ if
  and only if $\dpsi$ is $\frac{1}{\sigma}$-strongly smooth with
  respect to the dual norm $\dnorm{\cdot}$.
\end{lemma}
For the strong convexity implies strongly smooth direction see
\citet[Lemma~15]{shwartz07thesis}, and for the other direction see
\citet[Theorem 3]{kakade12regularization}.

\subsection{Regret Bound Proofs}\label{sec:proofs}
In this section, we prove Theorems \ref{thm:centered} and
\ref{thm:proximal} using Lemma~\ref{lem:strong_ftrl}.  Stating
these two analyses in a common framework makes clear exactly where the
``off-by-one'' issue arises for FTRL-Centered, and how assuming proximal $r_t$
resolves this issue.  The key tool is Lemma \ref{lem:smoothchange},
though for comparison we also provide a proof of
Theorem~\ref{thm:centered} for linearized functions from
Corollary~\ref{cor:strongconj} directly using strong smoothness.

\paragraph{General FTRL including FTRL-Centered (Proof of Theorem~\ref{thm:centered})}
In order to apply Lemma~\ref{lem:strong_ftrl}, we work to bound the stability terms in the sum in~\eqr{strongftrl}.
Fix a particular round $t$.  For Lemma~\ref{lem:smoothchange} take
$\ha(x) = h\zttm(x)$ and $\hb(x) = h\zttm(x) + f_t(x)$, so $x_t =
\argmin_x \ha(x)$, and by assumption $\hb$ is 1-strongly-convex \wrt
$\normtm{\cdot}$.
Then, applying Lemma~\ref{lem:smoothchange} to $\hb$ (with $x' =
x_{t+1}$), we have
$\hb(x_t) - \hb(x\ti) \le \h \dnormtm{g_t}^2$ for $g_t \in \partial f_t(x_t)$, and so
\begin{align*}
h\ztt(x_t) - h\ztt(x\ti) - r_t(x_t)
   &= \hb(x_t) + r_t(x_t) - \hb(x\ti) - r_t(x\ti) - r_t(x_t) \\
   & \le \h \dnormtm{g_t}^2
\end{align*}
where we have used the assumption that $r_t(x) \geq 0$ to drop the
$-r_t(x\ti)$ term.  We can now plug this bound into
Lemma~\ref{lem:strong_ftrl}.  However, we need to make one additional
observation: the choice of $r_T$ only impacts the bound by increasing
$r_{0:T}(\xs)$.  Further, $r_T$ does not influence any of the points
$x_1, \dots, x_T$ selected by the algorithm.  Thus, for analysis
purposes, we can take $r_T(x) = 0$ without loss of generality, and
hence replace $r_{0:T}(\xs)$ with $r_{0:T-1}(\xs)$ in the final bound.\qed

\paragraph{FTRL-Proximal (Proof of Theorem~\ref{thm:proximal})}
The key is again to bound the stability terms in the sum
in~\eqr{strongftrl}.  Fix a particular round $t$, and take $\ha(x) =
f\ttm(x) + r\ztt(x) = h\ztt(x) - f_t(x)$.  Since the $r_t$ are
proximal (so $x_t$ is a global minimizer of $r_t$) we have $x_t =
\argmin_x \ha(x)$, and $x\ti = \argmin_x \ha(x) + f_t(x)$.  Thus,
\begin{align}
h\ztt(x_t) - h\ztt(x\ti) - r_t(x_t)
  &\le h\ztt(x_t) - h\ztt(x\ti) && \text{Since $r_t(x) \ge 0$} \notag \\
  &= \ha(x_t) + f_t(x_t) - \ha(x\ti) - f_t(x\ti) \notag \\
  &\le \h \dnormt{g_t}^2, \label{eq:ppr}
\end{align}
where the last line follows by applying
\text{Lemma~\ref{lem:smoothchange}} to $\ha$ and $\hb(x) = \ha(x) +
f_t(x) = h\ztt(x)$.  Plugging into
Lemma~\ref{lem:strong_ftrl} completes the proof.  \qed

\paragraph{Primal-dual Analysis of General FTRL on Linearized Functions}
We give an alternative proof of Theorem~\ref{thm:centered} for linear
functions, $f_t(x) = g_t \cdot x$, using \eqr{strongconj}.  We remark
that in this case $x_t = \grad \rc\ttm(-g\ttm)$ (see
Lemma~\ref{lem:smooth} in Appendix~\ref{sec:convexproofs}).

By Lemma~\ref{lem:strongdual}, $\rc\ttm$ is
1-strongly-smooth with respect to $\dnormtm{\cdot}$, and so
\begin{equation} \label{eq:ss}
\rc\ttm(-g\tt) \leq  \rc\ttm(-g\ttm) - x_t\cdot g_t + \h \dnormtm{g_t}^2,
\end{equation}
and we can bound the per-round terms in \eqr{strongconj} by
\begin{align*}
  \rc\tt(-g\tt) - \rc\ttm(-g\ttm) +  x_t \cdot g_t
   &\le \rc\tt(-g\tt) -\rc\ttm(-g\tt)  + \h \dnormtm{g_t}^2 \\
   &\le  \h \dnormtm{g_t}^2,
\end{align*}
where we use \eqr{ss} to bound $-\rc\ttm(-g\ttm) + x_t \cdot g_t$, and
then used the fact that $\rc\ttm(-g\tt) \ge
\rc\tt(-g\tt)$ from Lemma~\ref{lem:basic}. \qed

\section{Additional Regularization Terms and Composite
  Objectives} \label{sec:composite}
In this section, we consider generalized FTRL algorithms where we
introduce an additional regularization term $\alpha_t \Psi(x)$ on each
round, where $\Psi$ is a convex function taking on only non-negative
values, and the weights $\alpha_t \ge 0$ for $t \ge 1$ are
non-increasing in $t$.  We further assume $\Psi$ and $r_0$ are both minimized
at $x_1$, and \WLOG $\Psi(x_1) = 0$ (as usual, additive constant terms do
not impact regret). We generalize our definition of
$h_t$ to $h_0(x) = r_0(x)$ and 
\begin{equation}\label{eq:genh}
h_t(x) = g_t\cdot x + \alpha_t \Psi(x) + r_t(x),
\end{equation}
so the FTRL update is
\begin{equation}\label{eq:genhupdate}
x\ti = \argmin_x h\ztt(x) = \argmin_x g\tt\cdot x + \alpha\tt \Psi(x) + r\ztt(x).
\end{equation}
In applications, generally the $g_t \cdot x_t$ terms
come from the linearization of a loss $\ell_t$, that is $g_t
= \partial \ell_t(x_t)$. Here $\ell_t$ is for example a loss function
measuring the prediction error on the $t$th training example for a
model parameterized by $x_t$.  (In fact, it is straightforward to
replace $g_t\cdot x$ with $\ell_t(x)$ in this section, but for
simplicity we assume linearization has been applied).

The $\Psi$ terms often encode a non-smooth regularizer, and might be
added for a variety of reasons. For example, the actual convex
optimization problem we are solving may itself contain regularization
terms. This is perhaps most clear in the case of applying an online
algorithm to a batch problem as in \eqr{batch}. For example:
\begin{itemize}
\item An $L_2$ penalty $\Psi(x) = \norm{x}_2^2$ might be added in
  order to promote generalization in a statistical setting, as in
  regularized empirical risk minimization.
\item An $L_1$ penalty $\Psi(x) = \norm{x}_1$ (as in the LASSO method)
  might be added to encourage sparse solutions and improve
  generalization in the high-dimensional setting ($n \gg T$).
\item An indicator function might be added by taking by taking
  $\Psi(x) = \indX(x)$ to force $x \in \X$ where $\X$ is a convex set
  of feasible solutions.
\end{itemize}
As discussed in Section~\ref{sec:feasible}, the case of $\Psi = \indX$
can be handled by our existing results. However, for other choices of
$\Psi$ it is generally preferable to only apply the linearization to
the part of the objective where it is necessary computationally; in
the $L_1$ case, given loss functions $\ell_t(x) + \lambda_1
\norm{x}_1$, we might partially linearize by taking $\hf_t(x) = g_t
\cdot x + \lambda_1 \norm{x}_1$, where $g_t \in \partial
\ell_t(x_t)$. 
Recall that the primary motivation for linearization was to reduce the
computation and storage requirements of the algorithm. Storing
and optimizing over $\ell\tt$ might be prohibitive; however, for
common choices of $\Psi$ and $r_t$, the optimization of
\eqr{genhupdate} can be represented and solved efficiently (often in
closed form). Thus, it is advantageous to consider such a composite
representation.

Further, even in the case of a feasible set $\Psi = \indX$, a careful
consideration of if and when $\Psi$ is linearized is critical to
understanding the connection between \MD and FTRL. In fact,
we will see that \MD \emph{always} linearizes the past penalties
$\alpha_{1:t-1} \Psi$, while with FTRL it is possible to avoid this
additional linearization as in \eqr{genhupdate} --- to make this distinction more clear, we will refer to the direct application of \eqr{genhupdate} as the \Natural algorithm. For $\Psi = \indX$
this gives rise to the distinction between ``lazy-projection'' and
``greedy-projection'' algorithms, as discussed in
Appendix~\ref{sec:lazyvgreedy}.  And for $\Psi(x) = \norm{x}_1$, this
distinction makes \Natural algorithms preferable to composite-objective \MD for
generating sparse models using $L_1$ regularization (see
Section~\ref{sec:l1}).

There are two types of regret bounds we may wish to prove in this
setting, depending on whether we group the $\Psi$ terms with the
objective $g_t$, or with the regularizer $r_t$. We discuss these
below.
\paragraph{In the objective} We may view the $\alpha_t \Psi(x)$ terms as part of the objective, in that we desire a bound on regret against the functions $f^\Psi_t(x) \equiv g_t \cdot x + \alpha_t \Psi(x)$, that is
\[
\Regret(\xs, f^\Psi) \equiv \sum_{t=1}^T f^\Psi_t(x_t)- f^\Psi_t(\xs).
\] 
This setting is studied
by~\citet{xiao09dualaveraging} and \citet{duchi10composite,duchi11adaptivejournal}, though in the less general setting where all $\alpha_t = 1$.
We can directly apply Theorem~\ref{thm:centered} or
Theorem~\ref{thm:proximal} to the $f^\Psi$ in this case, but this gives us
bounds that depend on terms like $\dnormt{g_t + \gp_t}^2$ where $\gp_t
\in \partial (\alpha_t \Psi)(x_t)$; this is fine for $\Psi = \indX$
since we can then always take $\gp_t = 0$ since $x_t \in \X$, but for
general $\Psi$ this bound may be harder to interpret.
Further, adding a fixed known penalty like $\Psi$ should intuitively make the
problem no harder, and we would like to demonstrate this in our
bounds.  

\paragraph{In the regularizer} We may wish to measure loss only against the functions $f_t(x) = g_t\cdot x$, that is,
\[
\Regret(\xs, g_t) \equiv \sum_{t=1}^T g_t\cdot x_t - g_t \cdot \xs,
\]
even though we include the terms $\alpha_t \Psi$ in the update of
\eqr{genhupdate}.  This approach is natural when we are only concerned
with regret on the learning problem, $f_t(x) = \ell_t(x)$, but wish to
add (for example) additional $L_1$ regularization in order to produce
sparse models, as in \citet{mcmahan13adclick}.

In this case we can apply Theorem~\ref{thm:centered} to $f_t(x)
\leftarrow g_t \cdot x$ and $r_t(x) \leftarrow r_t(x) + \alpha_t
\Psi(x)$, noting that if the original $r\ztt$ is strongly convex \wrt
$\normt{\cdot}$, then $r\ztt + \alpha\tt\Psi$ is as well, since
$\Psi$ is convex.  However, if $r_t$ is proximal, $r_t + \alpha_t \Psi$
generally will not be, and so a modified result is needed in
place of Theorem~\ref{thm:proximal}. The following theorem provides this as well as a bound on $\Regret(\xs, f^\Psi)$.

\begin{theorem}\label{thm:composite}\emph{\textbf{FTRL-Proximal Bounds for Composite Objectives}}
  Let $\Psi$ be a non-negative convex function minimized at $x_1$ with
  $\Psi(x_1) = 0$. Let $\alpha_t \ge 0$ be a non-increasing sequence
  of constants.
  Consider Setting~\ref{setting}, and define $h_t$ as in
  \eqr{genh}. Suppose the $r_t$ are chosen such that $h\ztt$ is
  1-strongly-convex \wrt some norm $\normt{\cdot}$, and further the
  $r_t$ are proximal, that is $x_t$ is a global minimizer of $r_t$.

  When we consider regret against $f^\Psi_t(x) = g_t \cdot x +
  \alpha_t \Psi(x)$, we have
  \begin{equation}\label{eq:inf}
     \Regret(\xs, f^\Psi) 
       \leq r\ztT(\xs) + \h \sum_{t=1}^T \dnormt{g_t}^2.
  \end{equation}

  When we consider regret against only the functions $f_t(x) = g_t
  \cdot x$, we have
  \begin{equation}\label{eq:noinf}
     \Regret(\xs, g_t) 
       \leq r\ztT(\xs) + \alpha_{1:T}\Psi(\xs) + \h \sum_{t=1}^T \dnormt{g_t}^2.
  \end{equation}
\end{theorem}
\begin{proof}
  The proof closely follows the proof of Theorem~\ref{thm:proximal} in
  Section~\ref{sec:proofs}, with the key difference that we use
  Corollary~\ref{cor:smoothchange2} in place of
  Lemma~\ref{lem:smoothchange}. We will use
  Lemma~\ref{lem:strong_ftrl} to prove both claims. First, observe
  that the stability terms $h\ztt(x_t) - h\ztt(x\ti)$ depend only on
  $h$, and so we can bound them in the same way in both cases.

  Take $\ha(x) = h\zttm(x) + r_t(x)$.  Since the $r_t$ are proximal
  (so $x_t$ is a global minimizer of $r_t$) we have $x_t = \argmin_x
  \ha(x)$, and $x\ti = \argmin_x \hb(x)$ where $\hb(x) = \ha(x) +
  g_t\cdot x + \alpha_t \Psi(x) = h\ztt(x)$.  Then, using
  Corollary~\ref{cor:smoothchange2} lets us replace \eqr{ppr} with
  \begin{align*}
    h\ztt(x_t) - h\ztt(x\ti) - r_t(x_t)
    &\le \h \dnormt{g_t}^2 + \alpha_t \Psi(x_t) - \alpha_t \Psi(x\ti).
  \end{align*}

  To apply Lemma~\ref{lem:strong_ftrl} we sum over $t$. Considering
  only the $\Psi$ terms, we have
  \begin{equation*}%
    \sum_{t=1}^T \alpha_t \Psi(x_t) - \alpha_t \Psi(x\ti)
    =  \alpha_1 \Psi(x_1) - \alpha_{T}\Psi(x_{T+1})
    +  \sum_{t=2}^T \alpha_t \Psi(x_t) - \alpha_{t-1} \Psi(x_t)
    \le 0,
  \end{equation*}
  since $\Psi(x) \ge 0$, $\alpha_t \le \alpha_{t-1}$, and $\Psi(x_1) =
  0$.  Thus, 
  \[
  \sum_{t=1}^T h\ztt(x_t) - h\ztt(x\ti) - r_t(x_t) \leq \h \sum_{t=1}^T  \dnormt{g_t}^2.
  \]
  Using this with Lemma~\ref{lem:strong_ftrl} applied to $f_t(x)
  \leftarrow g_t \cdot x + \alpha_t \Psi(x)$ and $r_t \leftarrow r_t$
  proves \eqr{inf}. For \eqr{noinf}, we apply
  Lemma~\ref{lem:strong_ftrl} taking $f_t(x) \leftarrow g_t \cdot x$
  and $r_t(x) \leftarrow \alpha_t \Psi(x) + r_t(x)$.
\end{proof}
For FTRL-Centered algorithms, Theorem~\ref{thm:centered} immediately
gives a bound for $\Regret(\xs, g_t)$. For the $\Regret(\xs, f^\Psi)$
case, we can prove a bound matching Theorem~\ref{thm:centered} using
arguments analogous to the above.

\section{\MD, FTRL-Proximal, and Implicit Updates}
\label{sec:md}

Recall Section~\ref{sec:clrogd} showed the equivalence between
constant learning rate \OGD and a fixed-regularizer FTRL algorithm.
This equivalence is well-known in the case where $r_t(x) = 0$ for $t \ge
1$, that is, there is a fixed stabilizing regularizer $r_0$
independent of $t$, and further we take $\X = \R^n$ (e.g.,
\citet{rakhlin08notes,hazan10survey,shwartz12online}).  Observe that
in this case FTRL-Centered and FTRL-Proximal coincide.  In this
section, we show how this equivalence extends to adaptive regularizers
(equivalently, adaptive learning rates) and composite objectives. This
builds on the work of \citet{mcmahan10equiv}, but we make some crucial
improvements in order to obtain an exact equivalence result for all
\MD algorithms. 

\paragraph{Adaptive \MD}
Even in the non-adaptive case, \MD can be expressed as a variety of
different updates, some equivalent but some not;\footnote{In
  particular, it is common to see updates written in terms of $\grad
  \rc(\ng)$ for a strongly convex regularizer $r$, based on the fact
  that $\grad \rc(-\ng) = \argmin_x \ng \cdot x + r(x)$ (see Lemma
  \ref{lem:smooth} in Appendix~\ref{sec:convexproofs}).} in
particular, the inclusion of the feasible set constraint $\indX$ gives
rise to distinct ``lazy projection'' vs ``greedy projection''
algorithms --- this issue is discussed in detail in
Appendix~\ref{sec:non_adaptive_md}.
To define the adaptive \MD family of algorithms we first define the Bregman divergence
with respect to a convex differentiable function\footnote{Certain properties of Bregman divergences require $\phi$ to be strictly convex, but it provides a convenient notation to define $\B_\phi(u, v)$ for any differentiable convex $\phi$.} $\phi$:
\[
\B_\phi(u, v) = \phi(u) - \big( \phi(v) + \grad \phi(v)\cdot(u-v)\big).
\]
The Bregman divergence is the difference at $u$ between $\phi$ and
$\phi$'s first-order Taylor expansion taken at $v$.  For example, if we take $\phi(u) = \norm{u}^2$, then $\B_\phi(u, v) =
\norm{u - v}^2$.

An adaptive \MD algorithm is defined by a sequence of continuously
differentiable incremental regularizers $r_0, r_1, \dots$, chosen so
$r\ztt$ is strongly convex.  From this, we define the time-indexed
Bregman divergence $\B_{r\ztt}$; to simplify notation we define $\B_t
\equiv \B_{r\ztt}$, that is,
\[ \B_t(u, v) = r\ztt(u) - \big( r\ztt(v) + \grad r\ztt(v)\cdot(u-v)\big).
\]
The adaptive \MD update is then given by
\begin{align}
\hx_1 &= \argmin_x r_0(x) \notag \\
\hx\ti &= \argmin_x\ g_t \cdot x + \alpha_t \Psi(x) + \B_t(x, \hx_t).\label{eq:md}
\end{align}
We use $\hat{x}$ to distinguish this update from an FTRL update we
will introduce shortly. Building on the previous section, we allow the
update to include an additional regularization term $\alpha_t
\Psi(x)$. As before, typically $g_t \cdot x$ should be viewed as a
subgradient approximation to a loss function $\ell_t$; it will become
clear that a key question is to what extent $\Psi$ is also linearized.

\MD algorithms were introduced in~\citet{nemirovski83} for
the optimization of a fixed non-smooth convex function, and
generalized to Bregman divergences by \citet{beck03md}.  Bounds for
the online case appeared in \citet{warmuth98}; a general treatment in
the online case for composite objectives (with a non-adaptive learning
rate) is given by \citet{duchi10composite}.
Following this existing literature, we might term the
update of \eqr{md} Adaptive Composite-Objective Online Mirror Descent;
for simplicity we simply refer to \MD in this work.

\paragraph{Implicit updates}
For the moment, we neglect the $\Psi$ terms and consider convex
per-round losses $\ell_t$.  While standard \OGD (or \MD)
linearizes the $\ell_t$ to arrive at the update $\hx\ti = \argmin_x\
g_t \cdot x_t + \B_t(x, \hx_t)$, we can define the alternative update
\begin{equation}\label{eq:implicit}
\hx\ti = \argmin_x\ \ell_t(x) + \B_t(x, \hx_t),
\end{equation}
where we avoid linearizing the loss $\ell_t$.  This is often referred to as
an implicit update, since for general convex $\ell_t$ it is no longer
possible to solve for $\hx\ti$ in closed form.  The implicit update
was introduced by~\citet{kivinen94exponentiated}, and has more
recently been studied by
\citet{kulis10implicit}.

Again considering the $\Psi$ terms, the \MD update of
\eqr{md} can be viewed as a partial implicit update: if the real loss
per round is $\ell_t(x) + \alpha_t \Psi(x)$, we linearize the
$\ell_t(x)$ term but not the $\Psi(x)$ term, taking $f_t(x) = g_t\cdot
x + \alpha_t \Psi(x)$.  Generally this is done for computational
reasons, as for common choices of $\Psi$ such as $\Psi(x) =
\norm{x}_1$ or $\Psi(x) = \indX(x)$, the update can still be solved in
closed form (or at least in a computationally efficient manner, e.g.,
by projection). However, while $\alpha_t \Psi$ is handled without
linearization, we shall see that echoes of the past $\alpha_{1:t-1}
\Psi$ are encoded in a linearized fashion in the current state
$\hx_t$.

\paragraph{On terminology}
In the unprojected and non-adaptive case, the \MD update $\hx\ti =
\argmin_x g_t\cdot x + \B_r(x, \hx_t)$ is equivalent to the FTRL
update $x\ti = \argmin_x g\tt\cdot x + r(x)$ (see
Appendix~\ref{sec:non_adaptive_md}).  In fact,
\citet[Sec. 2.6]{shwartz12online} refers to this update (with
linearized losses) explicitly as \MD.

In our view, the key property that distinguishes \MD from FTRL is that
for \MD, the state of the algorithm is exactly $\hx_t \in \R^n$, the
current feasible point.  For FTRL on the other hand, the state is a
different vector in $\R^n$, for example $g_{1:t}$ for \DA. The
indirectness of the FTRL representation makes it more flexible, since
for example multiple values of $g_{1:t}$ can all map to the same
coefficient value $x_t$.

\begin{figure}[t]
\begin{mdframed}
\vspace{-0.4in}
\begin{align*}
\intertext{\MD \vspace{-0.1in}}
\hx\ti &= \argmin_x g_t \cdot x + \alpha_t \Psi(x) + \B_{r\ztt}(x, \hx_t) 
\intertext{\MD as FTRL-Proximal \vspace{-0.1in}}
\hx\ti 
&= \argmin_x  g\tt \cdot x + \gp\ttm \cdot x + \alpha_t\Psi(x)   + r_0(x) + \sum_{s=1}^t \B_{r_s}(x, x_s)\\
&= \argmin_x  g\tt \cdot x + \gp\tt \cdot x   + r_0(x) + \sum_{s=1}^t \B_{r_s}(x, x_s)\\
& \text{where $\gp_s$ is a suitable subgradient from $\partial (\alpha_s \Psi)(x_{s+1})$}  
\end{align*}
\end{mdframed}
\caption{\MD as normally presented, and expressed as an equivalent FTRL-Proximal update.}\label{fig:mdisftrl}
\end{figure}

\subsection{\MD is an FTRL-Proximal Algorithm}
We will show that the \MD update of \eqr{md} can be expressed as the
FTRL-Proximal update given in Figure~\ref{fig:mdisftrl}.
In particular, consider a \MD algorithm defined by the choice of $r_t$
for $t \ge 0$.  Then, we define the FTRL-Proximal update
\begin{equation}\label{eq:ftrlmd}
x\ti = \argmin_x  g\tt \cdot x + \gp\ttm \cdot x + \alpha_t\Psi(x)  + 
\rp\ztt(x)
\end{equation}
for an appropriate choice $\gp_t \in \partial (\alpha_t \Psi)(x\ti)$ (given below),
where $\rp_t$ is
an incremental proximal regularizer defined in terms of $r_t$, namely
\begin{align*}
\rp_0(x) &\equiv r_0(x) \notag \\
\rp_t(x) &\equiv \B_{r_t}(x, x_t) = r_t(x) - \big(r_t(x_t) + \grad r_t(x_t) \cdot (x - x_t)\big) && \text{for $t \ge 1$}. \label{eq:equivproxreg}
\end{align*}
Note that $\rp_t$ is indeed minimized by $x_t$ and $\rp_t(x_t) = 0$. 
We require  $\gp_t \in \partial (\alpha_t \Psi)(x\ti)$ such that
\begin{equation}\label{eq:subgradprop}
g\tt + \gp\tt + \grad \rp\ztt(x_{t+1}) = 0.
\end{equation}
The dependence of $\gp_t$ on $x\ti$ is not problematic, as $\gp_t$ is not necessary to compute $x\ti$ using \eqr{ftrlmd}. To see (inductively) that we can always find a a $\gp_t$ satisfying \eqr{subgradprop},
note the subdifferential of the objective of \eqr{ftrlmd} at $x$ is
\begin{equation}\label{eq:ftrlmdsg}
  g\tt + \gp\ttm + \partial(\alpha_t \Psi)(x)
  + \grad \rp\ztt(x).
\end{equation}
Since $x\ti$ is a minimizer, we know $0$ is a subgradient,
which implies there must be a subgradient $\gp_t \in \partial
(\alpha_t \Psi)(x\ti)$ that satisfies \eqr{subgradprop}.  The fact we
use a subgradient of $\Psi$ at $x_{t+1}$ rather than $x_t$ is a
consequence of the fact we are replicating the behavior of a (partial)
implicit update algorithm.  

Finally, note the update 
\begin{equation}\label{eq:hftrlmd}
  x\ti = \argmin_x  g\tt \cdot x + \gp\tt \cdot x + \rp\ztt(x)
\end{equation}
is equivalent to \eqr{ftrlmd}, since Equations \eqref{eq:subgradprop}
and \eqref{eq:ftrlmdsg} imply 0 is in the subgradient of the objective
\eqr{ftrlmd} at the $x\ti$ given by \eqr{hftrlmd}. This update is exactly an FTRL-Proximal update on the functions $f_t(x) = (g_t + \gp_t)\cdot x$.

With these definitions in place, we can now state and prove the main
result of this section, namely the equivalence of the two updates
given in Figure~\ref{fig:mdisftrl}:
\begin{theorem}\label{thm:mdequiv}
  The \MD update of \eqr{md} and the FTRL-Proximal update
  of \eqr{ftrlmd} select identical points. 
\end{theorem}
\newcommand{\gph}{\hat{g}^{(\Psi)}}
\begin{proof}
  The proof is by induction on the hypothesis that $\hx_t =
  x_t$.  This holds trivially for $t=1$, so we proceed by assuming it
  holds for $t$.

  First we consider the $x_t$ selected by the FTRL-Proximal
  algorithm of \eqr{ftrlmd}.  Since $x_t$ minimizes this objective,
  zero must be a subgradient at $x_t$.  Letting $\gr_s = \grad
  r_s(x_s)$ and noting $\grad \rp_t(x) = \grad r_t(x) - \grad
  r_t(x_t)$, we have $ g\ttm + \gp\ttm + \grad r\zttm(x_t) - \gr\zttm =
  0$ following \eqr{ftrlmdsg}.  Since $x_t = \hx_t$ by induction
  hypothesis, we can rearrange and conclude
  \begin{equation}\label{eq:sgxt}
    -\grad r\zttm(\hx_t) = g\ttm  + \gp\ttm - \gr\zttm .
  \end{equation}
  For \MD, the gradient of the objective in \eqr{md} must
  be zero for $\hx\ti$, and so there exists a $\gph_t \in \partial
  (\alpha_t \Psi)(\hx\ti)$ such that
  \begin{align*}
    0
    & =g_t + \gph_t + \grad r\ztt(\hx\ti) - \grad r\ztt(\hx_t)\\
    &= g_t + \gph_t + \grad r\ztt(\hx\ti) - \grad r\zttm(\hx_t) -\gr_t
         && \text{IH and $\grad r_t(x_t) = \gr_t$}\\
    &= g_t + \gph_t + \grad r\ztt(\hx\ti) + g\ttm + \gp\ttm - \gr\zttm -\gr_t
         && \text{Using \eqr{sgxt}}\\
    &= g\tt + \gp\ttm + \gph_t + \grad r\ztt(\hx\ti) - \gr\ztt \\
    &= g\tt + \gp\ttm + \gph_t + \grad \rp\ztt(\hx\ti).
  \end{align*}
  The last line implies zero is a subgradient of the objective of
  \eqr{ftrlmd} at $\hx\ti$, and so $\hx\ti$ is a minimizer.  Since
  $r\ztt$ is strongly convex, this solution is unique and so $\hx\ti =
  x\ti$.
\end{proof}

\subsection{Comparing \MD to the \NaturalProx Algorithm, and the Application to $L_1$ Regularization}\label{sec:l1}
Since we can write \MD as a particular FTRL
update, we can now do a careful comparison to the direct application
of Section~\ref{sec:composite} which gives the \NaturalProx algorithm.  These two
algorithms are given in Figure~\ref{fig:mdvsnatural}, expressed in a way that
facilitates comparison.

\begin{figure}
\begin{mdframed}
\vspace{-0.25in}
\[
\begin{array}{rr*{3}{l@{\qquad }}}
\multicolumn{5}{l}{\text{\MD}} \rule{0pt}{4ex} \\
\quad \qquad \qquad  \hx\ti &= \argmin_x & g\tt \cdot x &
    + \ \ \gp\ttm \cdot x + \alpha_t\Psi(x)  & + \rp\ztt(x)  \\
\multicolumn{5}{l}{\text{\NaturalProx}} \rule{0pt}{4ex}\\
      x\ti &= \argmin_x & g_{1:t} \cdot x &
     + \ \ \alpha_{1:t} \Psi(x)  &  + \rp\ztt(x)\\
  &  &\  (A)  &  \qquad (B)  &  \qquad (C)  \rule{0pt}{4ex} \\

\end{array}
\]
\end{mdframed}
\caption{\MD expressed as an FTRL-Proximal algorithm compared to the \NaturalProx algorithm.
}\label{fig:mdvsnatural}
\end{figure}

Both algorithms use a linear approximation to the loss functions
$\ell_t$, as seen in column (A) of Figure~\ref{fig:mdvsnatural}, and
the same proximal regularization terms $(C)$.  The key difference is
in how the non-smooth terms $\Psi$ are handled: \MD approximates the
past $\alpha_s \Psi(x)$ terms for $s < t$ using a subgradient
approximation $\gp_s \cdot x$, keeping only the current $\alpha_t \Psi(x)$
term explicitly. In \NaturalProx, on the other hand, we represent the full
weight of the $\Psi$ terms exactly as $\alpha_{1:t}\Psi(x)$. That is,
\MD is applying significantly more linearization than \NaturalProx.

Why does this matter? As we will see in Section~\ref{sec:MDanalysis},
there is no difference in the regret bounds, even though intuitively
avoiding unnecessary linearization should be preferable. However,
there can be a substantial practical differences for some choices of
$\Psi$. In particular, we focus on the common and practically
important case of $L_1$ regularization, where we take $\Psi(x) =
\norm{x}_1$. Such regularization terms are often used to produce
sparse solutions ($x_t$ where many $x_{t,i} = 0$). Models with few
non-zeros can be stored, transmitted, and evaluated much more cheaply
than the corresponding dense models.

As discussed in \citet{mcmahan10equiv}, it is precisely the explicit
representation of the full $\alpha\tt \norm{x}_1$ terms that lets
\Natural produce much sparser
solutions when compared with the composite-objective \MD update with
$L_1$ regularization (equivalent to the FOBOS algorithm of
\citet{duchi09fobos}).  This argument also applies to Regularized \DA
(RDA, a Native FTRL-Centered algorithm); \citet{xiao09dualaveraging} presents experiments showing
the advantages of RDA for producing sparse solutions. In the remainder
of this section, we explore the application to $L_1$ regularization in
more detail, in order to illustrate the effect of the additional
linearization of the $\norm{x}_1$ terms used by \MD as compared to the
\NaturalProx algorithm.

Another way to understand this distinction is the previously mentioned difference in how the two algorithms maintain state. \MD has exactly one way
to represent a zero coefficient in the $i$th coordinate, namely
$\hx_{t,i} = 0$.  The FTRL representation is significantly more
flexible, since many state values, say any $g_{1:t,i} \in [-\lambda,
\lambda]$, can all correspond to a zero coefficient.  This means that
FTRL can represent both ``we have lots of evidence that $x_{t,i}$
should be zero'' (as $g_{1:t,i} = 0$ for example), as well as ``we
think $x_{t,i}$ is zero right now, but the evidence is very weak'' (as
$g_{1:t,i} = \lambda$ for example).  This means there may be a memory
cost for training FTRL, as $g_{1:t,i} \neq 0$ still needs to be stored
when $x_{t,i}=0$, but the obtained models typically provide much better
sparsity-accuracy tradeoffs \citep{mcmahan10equiv,mcmahan13adclick}.

This distinction is critical even in the non-adaptive
  case, and so we consider the simplest possible setting: a fixed
regularizer $r_0(x) = \frac{1}{2\eta}\norm{x}^2_2$ (with $r_t(x) = 0$
for $t \ge 1)$, and $\alpha_t \Psi(x) = \lambda \norm{x}_1$ for all
$t$.  The updates of Figure~\ref{fig:mdvsnatural} then simplify to:\\
\begin{minipage}{\textwidth} %
\begin{align}
\shortintertext{\MD}
x\ti &= \argmin_x&  g\tt \cdot x\quad
    &+  \gp\ttm \cdot x + \lambda \norm{x}_1  &+ \frac{1}{2\eta}\norm{x}^2_2
     \label{eq:mdl1} \\
\shortintertext{\Natural}
     x\ti &= \argmin_x&  g_{1:t} \cdot x \quad
     &+ t \lambda \norm{x}_1 &+ \frac{1}{2\eta}\norm{x}^2_2.
     \label{eq:ftrll1} 
\end{align}
\end{minipage}
The key point is the \Natural algorithm uses a much
stronger explicit $L_1$ penalty, $\alpha_{1:t} = t \lambda$ instead of
just $\alpha_t = \lambda$.

\paragraph{The closed-form update}
We can write the update of \eqr{mdl1} as a standard \MD update (that is, as an optimization over $f_t$ and a regularizer centered at the current $x_t$):
\begin{align}
x\ti
&= \argmin_x g_t \cdot x + \lambda \norm{x}_1
           + \frac{1}{2 \eta}\norm{x - x_t}_2^2 \notag \\
&= \argmin_x \big(g_t - \frac{x_t}{\eta}\big) \cdot x + \lambda \norm{x}_1
           + \frac{1}{2 \eta}\norm{x}_2^2. \label{eq:gdl1}
\end{align}
The above update decomposes on a per-coordinate basis.  Subgradient calculations show that for constants $a > 0$, $b \in \R$,
and $\lambda \ge 0$, we have
\begin{equation}\label{eq:l1soln}
\argmin_{x \in \R}
b\cdot x + \lambda \norm{x}_1 + \frac{a}{2} \norm{x}^2
 = \begin{cases}
0 & \text{when $\abs{b} \le \lambda$} \\
- \frac{1}{a} (b - \text{sign}(b) \lambda) & \text{otherwise.}
\end{cases}
\end{equation}
Thus, we can simplify \eqr{gdl1} to
\[
x\ti =  \begin{cases}
0
& \text{when $\abs{g_t - \frac{x_t}{\eta}} \le \lambda$} \\
x_t - \eta (g_t  -\lambda)
& \text{when $g_t - \frac{x_t}{\eta} > \lambda$
\quad (implying $x\ti < 0$)
} \\
x_t - \eta (g_t  +\lambda)
& \text{otherwise \quad(i.e., $g_t - \frac{x_t}{\eta} < -\lambda$ and $x\ti > 0$).}
\end{cases}
\]
In fact, if we choose $\gp_t \in \partial \lambda \norm{x\ti}_1$ as
\[
\gp_t = 
\begin{cases}
-\lambda      &\text{when $x\ti < 0$}\\
\lambda        & \text{when $x\ti > 0$}\\
x_t/\eta - g_t & \text{when $x\ti = 0$}
\end{cases}
\mbox{,}
\]
then \eqr{subgradprop} is satisfied, and the update becomes
\[
  x\ti = x_t - \eta \big(g_t + \gp_t\big)
\]
in all cases, showing how the implicit update can be re-written in
terms of a subgradient update using an appropriate subgradient
approximation at the \emph{next} point.

\begin{figure}
\begin{center}
    \includegraphics[width=2.7in]{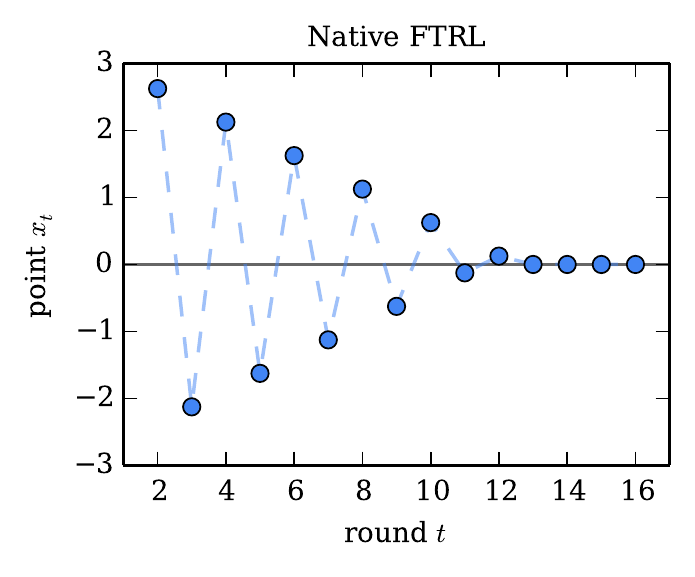}
    \includegraphics[width=2.7in]{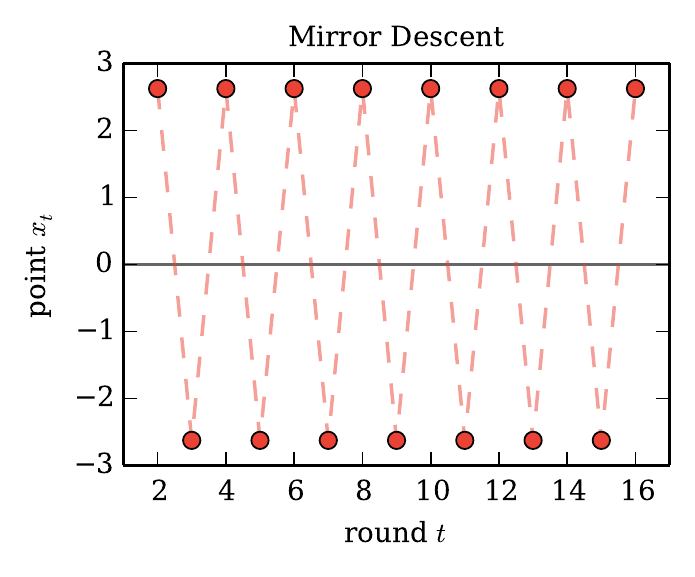}
\end{center}
\vspace{-0.3in}
\caption{The points selected by \Natural and \MD on the
  one-dimensional example, using $\alpha_t \Psi(x) =
  \h\norm{x}_1$. \Natural quickly converges to $x^* = 0$, but
  \MD oscillates indefinitely.}\label{fig:l1example}
\end{figure}

\paragraph{A One-Dimensional Example}
To illustrate the practical significance of the stronger explicit
$L_1$ penalty used by \Natural, we compare the updates of
\eqr{mdl1} and \eqr{ftrll1} on a simple one-dimensional example.  The
gradients $g_t$ satisfy $\norm{g_t}_2 \le G$, and we use a feasible
set of radius $R = 2G$. Both algorithms use the theory-recommended
fixed learning rate $\eta = \frac{R}{G\sqrt{T}}= \frac{2}{\sqrt{T}}$
(see Section~\ref{sec:applications}), against an adaptive adversary
that selects gradients $g_t$ as a function of $x_t$:
\begin{equation*}
g_t = \begin{cases}
-\h(G+\lambda) & \text{when $t = 1$}\\
-G & \text{when $t > 1$ and $x_t \le 0$}\\
G & \text{when $t > 1$ and $x_t > 0$\,.}
\end{cases} \label{eq:adversary}
\end{equation*}
Both algorithms select $x_1 = 0$, and since $g_1 = -\h(G+\lambda)$
both algorithms select $x_2 = (G-\lambda)/\sqrt{T}$. After this,
however, their behavior diverges:
\MD will indefinitely oscillate between $x_2$ and $-x_2$ for any
$\lambda < G$. On the other hand, FTRL learns that $x^* = 0$ is the
optimal solution after a constant number of rounds, selecting $x\ti
= 0$ for any $t > \frac{G}{2\lambda} + \h$. The details of this
example are worked out in Appendix~\ref{sec:l1example}

Figure~\ref{fig:l1example} plots the points selected by the algorithms
as a function of $t$, taking $G=11$, $T=16$, and $\lambda=0.5$.  This
example clearly demonstrates that, though \MD and \Natural
 have the same regret bounds, \Natural is much
more likely to produce sparse solutions and can also incur less actual
regret.

\subsection{Analysis of \MD as FTRL-Proximal}
\label{sec:MDanalysis}
Having established the equivalence between \MD and a particular
FTRL-Proximal update as given in Figure~\ref{fig:mdisftrl}, we now use the general analysis techniques for FTRL developed in this
work to prove regret bounds for any \MD algorithm.  This is
accomplished by applying the Strong FTRL lemma to the FTRL-Proximal
expression for \MD.

First, we observe that in the non-composite case (i.e., all $\alpha_t
=0$), then all $\gp_t = 0$, and we
can apply Theorem~\ref{thm:proximal} directly to \eqr{ftrlmd} for the loss functions $f_t(x) = g_t \cdot x$,  which gives
us
\begin{align*}
  \Regret(\xs, g_t)
    \leq \rp\ztT(\xs) + \h \sum_{t=1}^T \dnormt{g_t}^2
    &= \sum_{t=1}^T \B_{r_t}(\xs, x_t) + \h \sum_{t=1}^T \dnormt{g_t}^2.
\end{align*}
In the case of a composite-objective (nontrivial $\Psi$ terms,
including feasible set constraints such as $\indX$), we will arrive at
the same bound, but must refine our analysis somewhat to encompass the
partial implicit update of \eqr{ftrlmd}. This is accomplished in the
following theorem:

\begin{theorem}\label{thm:mdregret}
  We consider the \MD update of \eqr{md} under the same
  conditions as Theorem~\ref{thm:composite}.
  When we consider regret against $f^\Psi_t(x) = g_t \cdot x +
  \alpha_t \Psi(x)$, we have
  \begin{equation}\label{eq:infmd}
  \Regret(\xs, f^\Psi) \leq \rp\ztT(\xs) + \h \sum_{t=1}^T \dnormt{g_t}^2.
  \end{equation}
  When we consider regret against only the functions $f_t(x) = g_t
  \cdot x$, we have
  \begin{equation}\label{eq:noinfmd}
     \Regret(\xs, g_t) 
       \leq \rp\ztT(\xs) + \alpha_{1:T}\Psi(\xs) + \h \sum_{t=1}^T \dnormt{g_t}^2.
  \end{equation}
\end{theorem}

The bound of \eqr{infmd} matches
\citet[Prop. 3]{duchi11adaptivejournal},\footnote{ Mapping our
  notation to their notation, we have $f_t(x) = \ell_t(x) + \alpha_t
  \Psi(x) \Rightarrow \phi_t(x) = f_t(x) + \varphi(x)$ and $r_{1:t}(x)
  \Rightarrow \frac{1}{\eta}\psi_t(x)$.  Dividing their Update (4) by
  $\eta$ and using our notation, we arrive at exactly the update of
  \eqr{md}.
  We can take $\eta=1$ in their bound \WLOG.  Then, using the fact that
  $\psi_t$ in their notation is $r_{1:t}$ in our notation, we have
\begin{align*}
\B_{\psi\ti}(x^*, x\ti)  - \B_{\psi_t}(x^*, x\ti) &=
\psi\ti(x^*) - (\psi\ti(x\ti) + \grad\psi\ti(x\ti) \cdot (x - x\ti)) \\
& \qquad-\big( \psi_t(x^*) - (\psi_t(x\ti) + \grad\psi_t(x\ti) \cdot (x - x\ti))
\big)\\
&=r\ti(x^*) - \big( r\ti(x\ti) + \grad r\ti(x\ti) \cdot(x - x\ti)\big) \\
&=\B_{r\ti}(x^*, x\ti).
\end{align*}}
and also encompasses Theorem~2 of \citet{duchi10composite}.\footnote{We can take
  their $\alpha=1$ and $\eta=1$ \WLOG, and also assume our $\Psi(x_1) =
  0$.  Their $r$ is our $\Psi$, and the implicitly take our $\alpha_t
  = 1;$ their $\psi$ is our $r_0$ (with our $r_1, \dots, r_T$ all
  uniformly zero).  Thus, their bound amounts (in our notation) to:
  $\Regret \le \B_{r_0}(x^*, x_1) + \h\sum_{t=1}^T \dnorm{g_t}^2,$
  matching exactly the bound of our Theorem~\ref{thm:mdregret} (noting
  $\rp\ztt(x^*) = \B_{r_0}(x^*, x_1)$ in this case).}

\newcommand{\hh}{h}
\begin{proof}%
  First, by Theorem~\ref{thm:mdequiv}, this algorithm can equivalently
  be expressed as in \eqr{hftrlmd}. To simplify bookkeeping, we define
  \[
  \hf_t(x) = g_t \cdot x + \bPsi_t(x) \qqwhere
  \bPsi_t(x) = \alpha_t \Psi(x\ti) + \gp_t \cdot (x - x\ti),
  \]
  Then, the update
  \begin{equation}\label{eq:barfmd}
  x\ti = \argmin_x  \hf\tt(x) + \rp\ztt(x)
  \end{equation}
  is equivalent to \eqr{hftrlmd}, since the objectives differ only in
  constant terms. Note 
  \begin{equation}\label{eq:bpsifact}
    \bPsi_t(x\ti) = \alpha_t \Psi(x\ti) \qqand \forall x,\  \alpha_t \Psi(x) \ge \bPsi_t(x),
  \end{equation} where the second claim uses the convexity of $\alpha_t \Psi$.

  Observe that \eqr{barfmd} defines an FTRL-Proximal algorithm --- we can
  imagine the $\hf_t$ are computed by a black-box given $f_t$ which
  solves the optimization problem of \eqr{ftrlmd} in order to compute
  $\gp_t$. Thus, we can apply the Strong FTRL Lemma
  (Lemma~\ref{lem:strong_ftrl}). Again, the key is bounding the
  stability terms. Using $\hh_t(x) = \hf_t(x) + \rp_t(x)$, we have
\[
 \sum_{t=1}^T \hh_{1:t}(x_t) - \hh_{1:t}(x\ti) - r_t(x_t) 
\leq \sum_{t=1}^T \h \dnormt{g_t}^2 + \bPsi_t(x_t) - \bPsi_t(x\ti),
\]
using Corollary~\ref{cor:smoothchange2} as in
Theorem~\ref{thm:composite}.

We first consider regret against the functions $f^\Psi_t(x) = g_t
\cdot x + \alpha_t \Psi(x)$.  We can apply Lemma~\ref{lem:strong_ftrl}
to the functions $\hf_t$, yielding
\[ 
\Regret(\xs, \hf_t) \le  \rp\ztT(\xs) + \sum_{t=1}^T \h \dnormt{g_t}^2 +
    \bPsi_t(x_t) - \bPsi_t(x\ti).
\] 
However, this does not immediately yield a bound on regret against the $f^\Psi_t$.  While $\hf_t(\xs) \le f^\Psi_t(\xs)$, our actual loss $f^\Psi_t(x_t)$ could be larger
than $\hf_t(x_t)$.  Thus, in order to bound regret against $f^\Psi_t$, we must add terms
$f^\Psi_t(x_t) - \hf_t(x_t) = \alpha_t \Psi(x_t) - \bPsi_t(x_t)$.  This gives
\begin{align*}
\Regret(\xs, f^\Psi_t)
&\le \Regret(\xs, \hf_t) + \sum_{t=1}^T \alpha_t \Psi(x_t) - \bPsi_t(x_t)\\
&\leq \rp\ztT(\xs) + \sum_{t=1}^T \h \dnormt{g_t}^2
+ \bPsi_t(x_t) - \bPsi_t(x\ti) + \alpha_t \Psi(x_t) - \bPsi_t(x_t)\\
&= \rp\ztT(\xs) + \sum_{t=1}^T \h \dnormt{g_t}^2
 + \alpha_t \Psi(x_t)  - \alpha_t \Psi(x\ti),
\end{align*}
where the equality uses $\bPsi_t(x\ti) = \alpha_t \Psi(x\ti)$.
Recalling $\sum_{t=1}^T \alpha_t \Psi(x_t) - \alpha_t \Psi(x\ti) \le
0$ from the proof of Theorem~\ref{thm:composite} completes the proof of \eqr{infmd}.

For \eqr{noinfmd}, applying Lemma~\ref{lem:strong_ftrl}
with $r_t \leftarrow \bPsi_t + \rp_t$ and $f_t(x) \leftarrow g_t \cdot x$ yields 
\[ \Regret(\xs, g_t)
\leq \rp\ztT(\xs) + \bPsi\tt(\xs) + \sum_{t=1}^T \h \dnormt{g_t}^2 + \bPsi_t(x_t) - \bPsi_t(x\ti).
\]
\eqr{bpsifact} implies
$
 \bPsi_t(x_t) - \bPsi_t(x\ti) \le \alpha_t \Psi(x_t) - \alpha_t\Psi(x\ti),
$
and so the sum of these terms again vanishes. Finally, observing  $\bPsi\tt(\xs) \le \alpha\tt\Psi(\xs)$ completes the proof.
\end{proof}

\section{Conclusions}
\label{sec:conclusions}
Using a general and modular analysis, we have presented a unified view
of a wide family of algorithms for online convex optimization that
includes \DA, \MD, FTRL, and FTRL-Proximal, recovering and sometimes
improving regret bounds from many earlier works.  Our emphasis has
been on the case of adaptive regularizers, but the results recover
those for a fixed learning rate or regularizer as well.

\clearpage
\bibliography{../new,../my_pubs}

\clearpage
\appendix

\section{The Standard FTRL Lemma}
\label{sec:stdftrl}

The following lemma is a well-known tool for the analysis of FTRL
algorithms (see \cite{kalai03ftpl,hazan08extract}, \citet[Lemma
1]{hazan10survey}, and \citet[Lemma 2.3]{shwartz12online}):
\begin{lemma}[Standard FTRL Lemma]\label{lem:weak_ftrl}
  Let $f_t$ be a sequence of arbitrary (possibly non-convex) loss
  functions, and let $r_t$ be arbitrary non-negative regularization
  functions, such that $x_{t+1} = \argmin_x h\ztt(x)$ is well defined
  (recall $h\ztt(x) = f\tt(x) + r\ztt(x)$).  Then, the algorithm that
  selects these $x_t$ achieves
  \[
  \Regret(\xs) \leq r\ztT(\xs) + \sum_{t=1}^T f_{t}(x_t) - f_t(x\ti).
  \]
\end{lemma}
The proof of this lemma (e.g., \citet[Lemma 3]{mcmahan10boundopt})
relies on showing that if one could run the Be-The-Leader algorithm by
selecting $x_t = \argmin_x f_{1:t}(x)$ (which requires peaking ahead at
$f_t$ to choose $x_t$), then the algorithm's regret is bounded above by
zero.  

However, as we see by comparing Theorem~\ref{thm:proximal} and
\ref{thm:asc} (stated below), this analysis loses a factor of $1/2$ on
one of the terms.  The key is that being the leader is actually
\emph{strictly better} than always using the post-hoc optimal point, a fact
that is not captured by the Standard FTRL Lemma. To prove the Strong
FTRL Lemma, rather than first analyzing the Be-The-Leader algorithm
and showing it has no regret, the key is
to directly analyze the FTL algorithm (using a similar inductive
argument).  The proofs are also similar in that in both the basic
bound is proved first for regret against the functions $h_t$
(equivalently, the regret for FTL without regularization), and this
bound is then applied to the regularized functions and re-arranged to
bound regret against the $f_t$.

Using Lemma~\ref{lem:weak_ftrl}, we can prove the following weaker
version of Theorem~\ref{thm:proximal}:
\begin{theorem}\label{thm:asc} \emph{\textbf{Weak FTRL-Proximal Bound}}
  Consider Setting~\ref{setting}, and further suppose the $r_t$ are
  chosen such that $h\ztt = r\ztt + f\tt$ is 1-strongly-convex
  \wrt some norm $\normt{\cdot}$, and further
  the $r_t$ are proximal, that is $x_t$ is a global minimizer of
  $r_t$. Then, choosing any $g_t \in \partial f_t(x_t)$ on each round,
  for any $\xs \in \R^n$,
  \[
     \Regret(\xs) \leq  r\ztT(\xs) + \sum_{t=1}^T \dnormt{g_t}^2.
  \]
\end{theorem}

We prove Theorem~\ref{thm:asc} using strong smoothness via
Lemma~\ref{lem:smoothchange}.  An alternative proof that uses strong
convexity directly is also possible, closely following
\citet[Sec. 2.5.2]{shwartz12online}.
\paragraph{Proof of Theorem~\ref{thm:asc}}
Applying Lemma~\ref{lem:weak_ftrl}, it is sufficient to consider a fixed
$t$ and upper bound $f_t(x_t) - f_t(x\ti)$.  For this fixed $t$,
define a helper function $ \ha(x) = f\ttm(x) + r\ztt(x).  $ Observe
$x_t = \argmin_x \ha(x)$ since $x_t$ is a minimizer of $r_t(x)$, and
by definition of the update $x_t$ is a minimizer of $f\ttm(x) +
r\zttm(x)$.  Let $\hb(x) = \ha(x) + f_t(x) = h\ztt(x)$, so $\hb$ is
1-strongly convex with respect to $\normt{\cdot}$ by assumption, and
$x\ti = \argmin_x \hb(x)$.  Then, we have
\begin{align*}
 f_t(x_t) - f_t(x\ti)
  & \le g_t \cdot (x_t - x\ti)
    && \text{Convexity of $f_t$ and $g_t \in \partial f_t(x_t)$} \\
  &\le \dnormt{g_t} \normt{x_t - x\ti} && \text{Property of dual norms}\\
  &\le \dnormt{g_t} \dnormt{g_t} = \dnormt{g_t}^2. &&
  \text{Using \eqr{xchange} from Lemma~\ref{lem:smoothchange}}
\end{align*}
\qed Interestingly, it appears difficult to achieve a tight (up to
constant factors) analysis of non-proximal FTRL algorithms (e.g.,
FTRL-Centered algorithms like \DA) using Lemma~\ref{lem:weak_ftrl}.
The Strong FTRL Lemma, however, allowed us to accomplish this.

\section{Proofs For Section~\ref{sec:convex}}
\label{sec:convexproofs}
We first state a standard technical result (see \citet[Lemma
15]{shwartz07thesis}):
\begin{lemma} \label{lem:smooth} Let $\psi$ be 1-strongly convex \wrt
  $\norm{\cdot}$, so $\dpsi$ is 1-strongly smooth with respect to
  $\dnorm{\cdot}$.  Then,
  \begin{equation}\label{eq:sss}
    \norm{\grad \dpsi(z) - \grad \dpsi(z')} \le \dnorm{z - z'},
  \end{equation}
  and
  \begin{equation}\label{eq:conjupdate}
    \argmin_x g \cdot x + \psi(x)  = \grad \dpsi(-g).
  \end{equation}
\end{lemma}

In order to prove
Lemma~\ref{lem:smoothchange}, we first prove a somewhat easier result:
\begin{lemma}\label{lem:linearchange}
  Let $\ha: \R^n \rightarrow \R$ be strongly convex \wrt norm
  $\norm{\cdot}$, and let $x_1 = \argmin_x \ha(x)$, and define
  $\hb(x) = \ha(x) + b \cdot x$ for $b \in \R^n$.  Letting $x_2
  = \argmin_x \hb(x)$, we have
  \[
  \hb(x_1) - \hb(x_2) \le \h \dnorm{b}^2,
  \qqand
  \norm{x_1 - x_2} \leq \dnorm{b}.
  \]
\end{lemma}
\begin{proof}
We have
\[ -\phd(0) = - \max_x 0 \cdot x -\ha(x) = \min_x \ha(x) = \ha(x_1).
\]
and similarly,
\[
-\phd(-b) = - \max_x - b\cdot x - \ha(x) = \min_x b \cdot x + \ha(x) = b \cdot x_2 + \ha(x_2).
\]
Since $x_1 = \grad \phd(0)$ and $\phd$ is strongly-smooth
(Lemma~\ref{lem:strongdual}), \eqr{strongsmooth} gives
\begin{equation*}
\phd(-b) \leq
    \phd(0) + x_1 \cdot (-b-0) + \h \dnorm{b}^2.
\end{equation*}
Combining these facts, we have
\begin{align*}
  \ha(x_1) + b \cdot x_1 - \ha(x_2) - b \cdot x_2
    &= -\phd(0) + b \cdot x_1 + \phd(-b)\\
    &\le -\phd(0) + b \cdot x_1 + \phd(0) + x_1 \cdot (-b) + \h \dnorm{b}^2\\
    &= \h \dnorm{b}^2.
\end{align*}
For the second part, observe $\grad \phd(0) = x_1$, and $\grad
\phd(-b) = x_2$ and so $\norm{x_1 - x_2} \leq \dnorm{b}$, using both
parts of Lemma \ref{lem:smooth}.
\end{proof}

\begin{proof}[Proof of Lemma~\ref{lem:smoothchange}]
  We are given that $\hb(x) = \ha(x) + \psi(x)$ is 1-strongly convex
  \wrt $\norm{\cdot}$.  The key trick is to construct an alternative
  $\ha'$ that is also 1-strongly convex with respect to this same
  norm, but has $x_1$ as a minimizer.  Fortunately, this is easily
  possible: define $\ha'(x) = \ha(x) + \psi(x) - b \cdot x$, and note
  $\ha$ is 1-strongly convex \wrt $\norm{\cdot}$ since it differs from
  $\hb$ only by a linear function.  Since $b \in \partial \psi(x_1)$
  it follows that 0 is in $\partial(\psi(x) - b \cdot x)$ at $x =
  x_1$, and so $x_1 = \argmin \ha'(x)$.  Note $\hb(x) = \ha'(x) +
  b\cdot x$.  Applying Lemma~\ref{lem:linearchange} to $\ha'$ and
  $\hb$ completes the proof, noting for any $x'$ we have $\hb(x_1) -
  \hb(x') \leq \hb(x_1) - \hb(x_2)$.
\end{proof}

\begin{proof}[Proof of Corollary~\ref{cor:smoothchange2}]
  Let $x_2' = \argmin_x \ha(x) + \psi(x)$, so by
  Lemma~\ref{lem:smoothchange}, we have
  \begin{equation}\label{eq:pb}
    \ha(x_1) + \psi(x_1) - \ha(x_2') - \psi(x_2') \le \h \dnorm{b}^2,
  \end{equation}
  Then, noting $\ha(x_2') + \psi(x_2') \le \ha(x_2) + \psi(x_2)$ by
  definition, we have
  \begin{align*}
    \hb(x_1) - \hb(x_2) &=
    \ha(x_1) + \psi(x_1) + \Psi(x_1) - \ha(x_2) - \psi(x_2) - \Psi(x_2)\\
    &\le \ha(x_1) + \psi(x_1) + \Psi(x_1) - \ha(x_2') - \psi(x_2') - \Psi(x_2) \\
      &\le \h \dnorm{b}^2 + \Psi(x_1) - \Psi(x_2).
        && \text{Using \eqr{pb}.}
\end{align*}
Noting that $\hb(x_1) - \hb(x') \le \hb(x_1) - \hb(x_2)$ for any $x'$
completes the proof.
\end{proof}

\section{Non-Adaptive Mirror Descent and Projection}
\label{sec:non_adaptive_md}
Non-adaptive \MD algorithms have appeared in the literature
in a variety of forms, some equivalent and some not.  In this section
we briefly review these connections.  We first consider the
unconstrained case, where the domain of the convex functions is taken
to be $\R^n$, and there is no constraint that $x_t \in \X$.

\subsection{The Unconstrained Case}
Figure~\ref{fig:unconstrained} summarizes a set of equivalent expressions for the
unconstrained non-adaptive \MD algorithm.  Here we assume
$\RR$ is a strongly-convex regularizer which is differentiable on $\R^n$
so that the corresponding Bregman divergence $\B_\RR$ is defined.
Recall from Lemma~\ref{lem:smooth},
\begin{equation}\label{eq:gradRD}
\grad \RD(-g) = \argmin_x g \cdot x + \RR(x).
\end{equation}
We now prove that these updates are equivalent:

\begin{figure}
\begin{tabular}{|l|p{5.5cm}|p{5.5cm}|}
\hline
Explicit &
\tabeq{
\ng\ti &= \ng_t - g_t\\
x\ti &= \grad \RD(\ng\ti)
}
&
\tabeq{
\ng\ti &= \grad R(x_t)  - g_t\\
x\ti &= \grad \RD(\ng\ti)
}
\\
\hline
Implicit &
\multicolumn{2}{|c|}{
\tabeq{
x\ti &= \argmin_{x} g_t \cdot x + \B_R(x, x_t)
}
}
\\
\hline
FTRL &
\multicolumn{2}{|c|}{
\tabeq{
x\ti = \argmin_{x} g\tt\cdot x + R(x)
}
}
\\
\hline
\end{tabular}
\caption{ Four equivalent expressions for unconstrained \MD
  defined by a strongly convex regularizer $R$.  The top-right
  expression is from by~\citet{beck03md}, while the top-left
  expression matches the presentation of \citet[Sec
  2.6]{shwartz12online}.}\label{fig:unconstrained}
\end{figure}

\begin{theorem}
  The four updates in Figure~\ref{fig:unconstrained} are equivalent.
\end{theorem}
\begin{proof}
It is sufficient to prove three equivalences:
\vspace{-0.04in}
\begin{itemize} \itemsep -1pt
\item The two explicit formulations are equivalent.  For the
  right-hand version, we have $x_t = \grad \RD(\ng_t) = \argmin_x
  -\ng_t\cdot x + \RR(x)$ using \eqr{gradRD}. The optimality of
  $x_t$ for this minimization implies $0 = -\ng_t + \grad \RR(x_t)$,
  or $\grad \RR(x_t) = \ng_t$.
\item Explicit \Equiv FTRL: Immediate from \eqr{gradRD} and the
  fact that $\ng\ti = -g\tt$.

\item Implicit \Equiv FTRL:
That is,
\begin{align}
\hx\ti &= \argmin_x g_t \cdot x + \B_R(x, \hx_t) && \text{and}\label{eq:bmd}\\
x\ti &= \argmin_x g\tt \cdot x + \RR(x) \label{eq:bftrl}
\end{align}
are equivalent.  The proof is by induction on the
hypothesis $x_t = \hx_t$.  We must have from \eqr{bmd} and the
IH that $g_t + \grad \RR(\hx\ti) - \grad \RR(x_t) = 0$, and from
\eqr{bftrl} applied to $t-1$ we must have $\grad R(x_t) = -g\ttm$,
and so $\grad \RR(\hx\ti) = -g\tt$.  Then, we have the gradient of the
objective of \eqr{bftrl} at $\hx\ti$ is $g\tt + \grad \RR(\hx\ti) =
0$, and since the optimum of \eqr{bftrl} is unique, we must have
$\hx\ti = x\ti$.  The same general technique is used to prove the more
general result for adaptive composite \MD in Theorem~\ref{thm:mdequiv}.
\end{itemize}
\vspace{-0.1in}
\end{proof}

\subsection{The Constrained Case:  Projection onto $\X$}
\label{sec:lazyvgreedy}
Even in the non-adaptive case (fixed $\RR$), the story is already more
complicated when we constrain the algorithm to select from a convex set $\X$. For this section we take $\RR(x) = r(x) + \indX(x)$
where $r$ is continuously differentiable on $\dom \indX = \X$.

In this setting, the two explicit algorithms given in the previous
table are, in fact, no longer equivalent.
Figure~\ref{fig:constrained} gives the two resulting families of
updates. The classic \MD algorithm corresponds to the right-hand
column, and follows the presentation of~\citet{beck03md}. This
algorithm can be expressed as a greedy projection, and when $r(x) =
\frac{1}{2 \eta}\norm{x}_2^2$ gives a constant learning rate version
of the projected \OGD algorithm of \citet{zinkevich03giga}. The Lazy
column corresponds for example to the ``\OGD with lazy projections''
algorithm \citep[Cor. 2.16]{shwartz12online}.

The relationship to these projection algorithms is made explicit by
the last row in the table.  We define the projection
operator onto $\X$ with respect to Bregman divergence $\B_r$ by
\[ \projX^r(u) \equiv \argmin_{x \in \X} \B_r(x, u).
\]
Expanding the definition of the Bregman divergence, dropping terms
independent of $x$ since they do not influence the $\argmin$, and
replacing the explicit $x \in \X$ constraint with an $\indX$ term in
the objective, we have the equivalent expression
\begin{equation}\label{eq:altproj}
   \projX^r(u) = \argmin_{x} r(x) - \grad r(u) \cdot x + \indX(x).
\end{equation}
The names Lazy and Greedy come from the manner in which the projection
is used. For Lazy-Projection, the state of the algorithm is simply
$g\tt$ which can be updated without any need for projection;
projection is applied lazily when we need to calculate $x\ti$. For the
Greedy-Projection algorithm on the other, the state of the algorithm
is essentially $x_t$, and in particular $u\ti$ cannot be calculated
without knowledge of $x_t$, the result of greedily applying projection
on the previous round. If the $g_t$ are really linear
approximations to some $f_t$, however, a projection is needed on each
round for both algorithms to produce $x_t$ so $g_t \in \partial
f_t(x_t)$ can be computed.

Both the Lazy and Greedy families can be analyzed (including in the
more general adaptive case) using the techniques introduced in this
paper.
The Lazy family corresponds to the \Natural update of
Section~\ref{sec:composite}, namely
\[
x\ti = \argmin_x g\tt\cdot x + \indX(x) + r\ztt(x),
\]
which we encode as a single fixed non-smooth penalty $\Psi = \indX$
which arrives on the first round: $\alpha_1 = 1$ and $\alpha_t = 0$
for $t>1$.

The Greedy-Projection \MD algorithms, on the other hand, can be thought of us
receiving loss functions $g_t \cdot x + \indX(x)$ on each round: that
is, we have $\alpha_t = 1$ for all $t$.  This family is analyzed using the techniques from Section~\ref{sec:md}.  
In this setting, embedding $\indX(x)$ inside
$\RR$ can be seen as a convenience for defining $\grad\RD$,
\begin{equation}\label{eq:gradRDX}
\grad \RD(-g) = \argmin_x g \cdot x + r(x) + \indX(x).
\end{equation}
We have the following equivalence results:

\begin{figure}
\begin{tabular}{|l|p{5cm}|p{6cm}|}
\hline
    & Lazy %
    & Greedy %
    \\
\hline
Explicit &
\tabeq{
\ng\ti &= \ng_t - g_t\\
x\ti &= \grad \RD(\ng\ti)
}
&
\tabeq{
\ng\ti &= \grad r(x_t)  - g_t\\
x\ti &= \grad \RD(\ng\ti)
}
\\
\hline
Implicit &
&
\tabeq{
&x\ti = \\
&\ \argmin_{x} g_t \cdot x + \B_r(x, x_t) + \indX(x)
}\\
\hline
FTRL &
{$\begin{aligned}
x\ti = \argmin_{x} g\tt\cdot x + R(x)
\end{aligned}$}
&
\rule{0pt}{4.1ex}
\tabeq{
& x\ti = \\
&\ \argmin_x \big(g\tt + \gp\ttm\big)\cdot x + R(x)
}
\\
\hline
\parbox{2cm}{Projection}&
{$\begin{aligned}
u\ti &= \argmin_x g\tt\cdot x + r(x)\\
x\ti &= \projX^r(u\ti)
\end{aligned}$}
&
{$\begin{aligned}
u\ti &= \argmin_u g_t \cdot u + \B_r(u, x_t)\\
     &=\grad r^\star(\grad r(x_t) - g_t)\\
x\ti &= \projX^r(u\ti)
\end{aligned}$}
\\
\hline
\end{tabular}
\caption{The Lazy and Greedy families of \MD algorithms,
  defined via $\RR(x) = r(x) + \indX(x)$, where $r$ is a differentiable
  strongly-convex regularizer.  These families are not equivalent, but
  the different updates in each column are equivalent.
} \label{fig:constrained}
\end{figure}

\begin{theorem} The Lazy-Explicit, Lazy-FTRL, and Lazy-Projection
  updates from the left column of Figure~\ref{fig:constrained} are
  equivalent.
\end{theorem}
\begin{proof}
  First, we show Lazy-Explicit is equivalent to Lazy-FTRL.  Iterating
  the definition of $\ng\ti$ in the explicit version gives $\ng\ti =
  -g\tt$, and so the second line in the update becomes exactly $x\ti =
  \argmin_x g\tt \cdot x + R(x)$.

  Next, we show that Lazy-Projection is equivalent to the
  Lazy-Explicit update.  Optimality conditions for the minimization
  that defines $u\ti$ imply $\grad r(u\ti) = -g\tt$.  Then, the second
  equation in the Lazy-Projection update becomes
  \begin{align*}
    x\ti
    &= \projX^r(u\ti)
    = \argmin_{x} r(x) - \grad r(u\ti) \cdot x + \indX(x)
      && \text{Using \eqr{altproj}.}\\
    &= \argmin_{x} g\tt \cdot x + r(x) + \indX(x),
      && \text{Since $\grad r(u\ti) = -g\tt$.}
  \end{align*}
  which is exactly the Lazy-FTRL update (recalling $\RR(x) = r(x) +
  \indX(x)$).
\end{proof}

\begin{theorem} The Explicit, Implicit, FTRL, and Projected updates in
  the ``Greedy'' column of Figure~\ref{fig:constrained} are equivalent.
\end{theorem}
\paragraph{Proof} We prove the result via the following chain of equivalences:
\begin{itemize}
\item Greedy-Explicit \Equiv Greedy-Implicit (c.f. \citet[Prop
  3.2]{beck03md}).  We again use $\hx$ for the points selected by the
  implicit version,
  \begin{align*}
  \hx\ti &= \argmin_x g_t \cdot x + \B_r(x, x_t) + \indX(x) \\
         &= \argmin_x g_t \cdot x + r(x) - \grad r(x_t)\cdot x + \indX(x),
  \end{align*}
  where we have dropped terms independent of $x$ in the $\argmin$.
  On the other hand, plugging in the definition of $\ng\ti$, the
  explicit update is
  \begin{equation}\label{eq:explicit2}
  x\ti = \argmin_x - (\grad r(x_t) - g_t)\cdot x + r(x) + \indX(x),
  \end{equation}
  which is equivalent.

\item Greedy-Implicit \Equiv Greedy-FTRL: 
This is a special case of Theorem~\ref{thm:mdequiv}, taking $r_0 \leftarrow r + \indX$, $r_t(x) = \rp_t(x) = 0$ for $t \ge 1$, and $\alpha_t \Psi(x) = \indX(x)$ for $t \ge 1$.

\item When $\indX = \indX$, Projection is equivalent to the
  Greedy-Explicit expression.  First, note we can re-write the
  Greedy-Projection update as
  \begin{align*}
    u\ti &= \argmin_u -(\grad r(x_t) - g_t)\cdot u + r(u) \\
    x\ti &= \argmin_{x \in \X} \B_r(x, u\ti).
  \end{align*}
  Optimality conditions for the first expression imply $\grad r(u\ti)
  = \grad r(x_t) - g_t$.  Then,  the second update becomes
  \begin{align*}
    x\ti &= \projX^r(u\ti) \\
         &= \argmin_{x} r(x) - \grad r(u\ti) \cdot x + \indX(x)
           && \text{Using \eqr{altproj}.} \\
         &= \argmin_{x} r(x) - (\grad r(x_t) - g_t) \cdot x + \indX(x),
           && \text{Since $\grad r(u\ti) = \grad r(x_t) - g_t$.}
   \end{align*}
   which is equivalent to the Greedy-Explicit update, e.g.,
   \eqr{explicit2}.\qed
\end{itemize}

\section{Details for the One-Dimensional $L_1$ Example}
\label{sec:l1example}
In this section we provide details for the one-dimensional example presented in Section~\ref{sec:l1}.
Suppose gradients $g_t$ satisfy $\norm{g_t}_2 \le G$, and we use a
feasible set of radius $R = 2G$, so the theory-recommended fixed
learning rate is $\eta = \frac{R}{G\sqrt{T}}= \frac{2}{\sqrt{T}}$ (see
Section~\ref{sec:applications}).

We first consider the behavior of \MD: we construct the example so
that the algorithm oscillates between two points, $\hx$ and $-\hx$
(allowing the possibility that $\hx = -\hx = 0$).  In fact, given
alternating gradients of $+G$ and $-G$, in such an oscillation the
distance one update takes us must be $\eta(G - \lambda)$, assuming
$\lambda < G$.  Thus, we can cause the algorithm to oscillate between
$\hx = (G-\lambda)/\sqrt{T}$ and $-\hx$.  We assume an initial $g_1 =
-\h(G+\lambda)$, which gives us $x_2 = \hx$ for both \MD and FTRL when
$x_1 = 0$.

\newcommand{\bfrac}[2]{\left(\frac{#1}{#2}\right)}

This construction implies that for any constant $L_1$ penalty $\lambda
< G$, \MD will never learn the optimal solution $x^*=0$
(note that after the first round, we can view the $g_t$ as being for
example the subgradients of $f_t(x) = G\norm{x}_1$).  The points $x_t$
selected by \MD, the gradients, and the subgradients of the $L_1$
penalty are given by the following table:

\vspace{0.1in}
\begin{tabular}{l|cccccc}
$t$     & 1 & 2 & 3 & 4 & 5 &$\cdots$\\
\hline
$g_t$   & $g_1$ & $G$ & $-G$ &  $G$ & $-G$ & $\cdots$\\
$x_t$   & $0$ & $\hx$ & $-\hx$ & $\hx$ & $-\hx$ & $\cdots$\\
$\gp_t$ & $\lambda$ & $-\lambda$ & $\lambda$ & $-\lambda$& $\lambda$ & $\cdots$
\end{tabular}

\vspace{0.1in}
\noindent
While we have worked from the standard \MD update,
\eqr{gdl1}, it is instructive to verify the
FTRL-Proximal representation is indeed equivalent.
For example, using the values from the table, for $x_5$ we have
\begin{align*}
x_5
&= \argmin_x  g_{1:4} \cdot x + \gp_{1:3}\cdot x + \lambda \norm{x}_1   + \frac{1}{2\eta}\norm{x}^2_2\\
&= \argmin_x\, (g_1 + G) \cdot x + \lambda \cdot x + \lambda \norm{x}_1   + \frac{1}{2\eta}\norm{x}^2_2
 = -\frac{G - \lambda}{\sqrt{T}} = -\hx,
\end{align*}
where we solve the argmin by applying \eqr{l1soln} with
$b = g_1 + G + \lambda$.%

Now, contrast this with the FTRL update of \eqr{ftrll1}; we can solve
this update in closed form using \eqr{l1soln}.  First, note that FTRL
will not oscillate in the same way, unless $\lambda = 0$.  We have
that $x\ti = 0$ whenever $\abs{g_{1:t}} < t \lambda$.
Note that $g_{1:t}$ oscillates between $g_{1:t} = g_1 = -\h(G +
\lambda)$ on odd rounds $t$, and $g_{1:t} = g_1 + G = \h G - \h
\lambda$ on even rounds.  Since the magnitude of $g_{1:t}$ is larger
on odd rounds, if we have $\h (G + \lambda) \le t \lambda$ then $x\ti$
will always be zero; re-arranging, this amounts to $\lambda \ge
\frac{G}{2t - 1}$.  Thus, as with \MD, we need $\lambda \ge G$ to have
$x_2 = 0$ (plugging in $t=1$) but on subsequent rounds a \emph{much}
smaller $\lambda$ is sufficient to produce sparsity.  In the extreme
case, taking $\lambda = G/(2T - 1)$ is sufficient to ensure $x_T = 0$,
whereas we need a $\lambda$ value almost $2T$ times larger in order to
get $x_T = 0$ from \MD.

\end{document}